\documentclass{article}

\usepackage[preprint,nonatbib]{nips_2018}

\usepackage[utf8]{inputenc} 
\usepackage[T1]{fontenc}    
\usepackage{hyperref}       
\usepackage{url}            
\usepackage{booktabs}       
\usepackage{nicefrac}       
\usepackage{microtype}      

\usepackage{amssymb}
\usepackage{amsmath}
\usepackage{amsthm}
\usepackage[linesnumbered,lined,ruled]{algorithm2e}
\usepackage{wrapfig}
\usepackage{subfig}
\usepackage{graphicx}
\usepackage{multirow}
\usepackage{capt-of}

\newtheorem{theorem}{Theorem}
\newtheorem{lemma}{Lemma}

\newtheorem{definition}{Definition}

\newcommand{\matx}{\mathbf{X}}
\newcommand{\maty}{\mathbf{Y}}
\newcommand{\matz}{\mathbf{Z}}
\newcommand{\vecx}{\mathbf{x}}
\newcommand{\vecy}{\mathbf{y}}
\newcommand{\vecz}{\mathbf{z}}

\newcommand{\doms}{\mathcal{S}}
\newcommand{\domt}{\mathcal{T}}
\newcommand{\domd}{\mathcal{D}}

\newcommand{\setb}{\mathfrak{B}}

\newcommand{\seti}{\mathcal{I}}
\newcommand{\setk}{\mathcal{K}}

\newcommand{\loss}{\mathcal{L}}

\newcommand{\lossinter}{\mathcal{L}_{\mathrm{inter}}}
\newcommand{\lossintra}{\mathcal{L}_{\mathrm{intra}}}
\newcommand{\lossfunwr}{\mathcal{L}_{\mathrm{WR}}}

\newcommand{\diswr}{d_{\mathrm{W}}}
\newcommand{\disj}{d_{\mathrm{J}}}

\newcommand{\fr}{\mathbb{R}}
\DeclareMathOperator*{\fe}{\mathbb{E}}
\DeclareMathOperator*{\fp}{\mathbb{P}}

\newcommand{\risk}{\mathrm{R}}

\newcommand{\nnc}{h_\mathrm{NN}}

\title{Unsupervised Domain Adaptive Re-Identification: Theory and Practice}

\author{
Liangchen Song$^{12}$\thanks{Equally contribution}\quad Cheng Wang$^{23*}$ 
\quad Lefei Zhang${}^{1}$\quad Bo Du${}^{1}$ \\
\textbf{\quad Qian Zhang${}^{2}$\quad Chang Huang${}^{2}$
\quad Xinggang Wang${}^{3}$}
\\
${}^{1}$Wuhan University\quad ${}^{2}$Horizon Robotics\quad ${}^{3}$ Huazhong Univ. of Science and Technology  \\
\texttt{\{wangcheng,xgwang\}@hust.edu.cn,\{zhanglefei,remoteking\}@whu.edu.cn}  \\ 
\texttt{\{liangchen.song,qian01.zhang,chang.huang\}@horizon.ai}  \\
}

\begin{document}

\maketitle

\begin{abstract}
We study the problem of unsupervised domain adaptive re-identification (re-ID) which is an active topic in computer vision but lacks a theoretical foundation. We first extend existing unsupervised domain adaptive classification theories to re-ID tasks. Concretely, we introduce some assumptions on the extracted feature space and then derive several loss functions guided by these assumptions. To optimize them, a novel self-training scheme for unsupervised domain adaptive re-ID tasks is proposed. It iteratively makes guesses for unlabeled target data based on an encoder and trains the encoder based on the guessed labels. Extensive experiments on unsupervised domain adaptive person re-ID and vehicle re-ID tasks with comparisons to the state-of-the-arts confirm the effectiveness of the proposed theories and self-training framework. Our code is available \href{https://github.com/LcDog/DomainAdaptiveReID}{on GitHub}.
\end{abstract}

\section{Introduction}
To re-identify a particular is to identify it as (numerically) the same particular as one encountered on a previous occasion\cite{plantinga1961things}. 
Image/video re-identification (re-ID) is a fundamental problem in computer vision and re-ID techniques serve as an indispensable tool for numerous real life applications, for instance, person re-ID for public safety \cite{Zheng2016Oct}, travel time measurement via vehicle re-ID \cite{coifman1998vehicle}.
The key component of re-ID tasks is the notion of identity, 
which makes re-ID tasks quite different from traditional classification tasks in the following ways:
Firstly, determining the label involves two samples, i.e., there is no label when an individual sample is given;
secondly, in re-ID tasks the samples in test sets belong to a previously unseen identity while in classification tasks the test samples all fall into a known class. 
Take the person re-ID task as an example, our target is to spot a person of interest in an image set, which do not have a specific class and is not accessible in the training set.

In many practical situations, we face the problem that the training data and testing data are in different domains.
Going back to the person re-ID example, data from a new camera is placed in a new environment, i.e., a new domain is added, which are too costly and impractical to be annotated, a serviceable re-ID model should have a satisfactory accuracy on unlabeled data.
Unsupervised domain adaptation means that learning a model for target domain when given both a fully annotated source dataset and an unlabeled target dataset. 
Existing algorithms for unsupervised domain adaptive re-ID tasks typically learn domain-invariant representation or generate data for target domain through some newly designed networks, which are practical solutions but lack theoretical support \cite{Li2018Apr, Wang2018Mar, Deng2017Nov}.
Meanwhile, current theoretical analysis of unsupervised domain adaptation are only concerned with classification tasks \cite{BenDavid2010,BenDavid2014,Mansour2009Feb}, which is not suitable for re-ID tasks. A theoretical guarantee of the domain adaptive re-ID problem is in need.

In this paper, we first theoretically analyze unsupervised domain adaptive re-ID tasks based on \cite{BenDavid2014}, in which three assumptions are introduced for classification. 
It is assumed that the source domain and the target domain share a same label space in \cite{BenDavid2014}. However, in re-ID tasks, the notion of label is defined for pairwise data and the label indicates a data pair belongs to a same ID or not. We adapt the three assumptions for the input space of pairwise data.
Moreover, instead of imposing the assumptions on the raw data as \cite{BenDavid2014}, we assume the resemblance between the feature space of two domains. 
The first assumption is that the criteria of classifying feature pairs is the same between two domains, which is referred to as covariate assumption. The second one is Separately Probabilistic Lipschitzness, indicating that the feature pairs can be divided into clusters. And the last assumption is weight ratio, concerning the probability of existing a repeated feature among all the features from the two domains. 
Based on the three assumptions, we show the learnability of unsupervised domain adaptive re-ID tasks.
Moreover, since our guarantee is built on well extracted features from real images, the encoder, i.e. feature extractor, is trained via a novel self-training framework, which is originally proposed for NLP tasks \cite{mcclosky2006effective,mcclosky2006reranking}.
Concretely, we iteratively refine the encoder by making guesses on unlabeled target domain and then train the encoder with these samples.
In the light of the mentioned assumptions, we propose several loss functions on the encoder and samples with guessed label. And the problem of selecting which sample with guessed label to train is optimized by minimizing the proposed loss functions.  
For the Separately Probabilistic Lipschitzness assumption, we wish to minimize the intra-cluster and inter-cluster distance. Then the sample selecting problem is turned into data clustering problem and minimizing loss functions is transformed into finding a distance metric for the data.
Also, another metric for Weight Ratio is designed. After combining the two metrics together, we have a distance evaluating the confidence of the guessed labels. 
Finally, the DBSCAN clustering method \cite{coifman1998vehicle} is employed to generate data clusters according to a threshold on the distance. 
With pseudo-labels on selected data cluster from target domain, the encoder is trained with triplet loss \cite{Weinberger2009}, which is effective for re-ID tasks.

We carry out experiments on diverse re-ID tasks, which demonstrate the priority of our framework.
First the well studied person re-ID task is tested and we show the adaptation results between two large scale datasets, i.e. Market-1501 \cite{Zheng2015Dec} and DukeMTMC-reID \cite{Ristani2016}. 
Then we evaluate our algorithm on vehicle re-ID task, in which larger datasets VeRi-776 \cite{Liu7553002} and PKU-VehicleID \cite{liu2016deep} are involved.
To sum up, the structure of our paper is shown in Figure \ref{fig:structure} and our contributions are as follows:
\vspace{-5pt}
\begin{itemize}
    \item We introduce the theoretical guarantees of unsupervised domain adaptive re-ID based on \cite{BenDavid2014}. A learnability result is shown under three assumptions that concerning the feature space. To the best of our knowledge, our paper is the first theoretical analysis work on domain adaptive re-ID tasks.
    \item We theoretically turn the goal of satisfying the assumptions into tractable loss functions on the encoder network and data samples.
    \item A self-training scheme is proposed to iteratively minimizing the loss functions. Our framework is applicable to all re-ID tasks and the effectiveness is verified on large-scale datasets for diverse re-ID tasks. 
\end{itemize}

\vspace{-5pt}
\subsection{Related work}

Unsupervised domain adaptation has been widely studied for decades and the algorithms are divided into four categories in a survey \cite{margolis2011literature}. 
Using the notions in the survey, our proposed method can be viewed as a combination of feature representation and self-training. 
Nevertheless, recently unsupervised domain adaptation is widely studied for the person re-ID task.

\vspace{-5pt}
\paragraph{Unsupervised domain adaptation and feature representation.} 
Feature representation based methods try to find a latent feature space shared between domains. In \cite{satpal2007domain}, they wish to minimize a distance between means of the two domains. In a more general manner, \cite{pan2008transfer} and \cite{chen2009extracting} try to find a feature space in which the source and target distributions are similar and the statistic Maximum Mean Discrepancy (MMD) is employed. Also, \cite{ganin2016domain} utilize features that cannot discriminate between source and target domains. Our method and these methods have a same intuition that some features from the source and target domain are generalizable. However, unlike these methods, the process of approximating the intuition in our method is in an iterative manner and we do not directly optimize on the distribution of target domain features. 

\vspace{-5pt}
\paragraph{Unsupervised domain adaptation and self-training.}
Self-training methods make guesses on target domain and iteratively refine the guesses and are closely related to the Expectation Maximization (EM) algorithm \cite{nigam2006semi}. In \cite{tan2009adapting}, they increase the weight on the target data at each iteration, which is actually altering the relative contribution of source and target domains. A more similar work is \cite{bacchiani2003unsupervised}, in which the model is initially trained on source domain, and then the top-1 recognition hypotheses on the target domain are used for adapting their language model. In our algorithm, we do not guess the labels since different re-ID datasets have totally different labels (identities) and instead we perform clustering on the data. 

\vspace{-5pt}
\paragraph{Unsupervised domain adaptive person re-ID.} 
Due to the rapid development of person re-ID techniques, some useful unsupervised domain adaptive person re-ID methods are proposed. \cite{peng2016unsupervised} adopt a multi-task dictionary learning scheme to learn a view-invariant representation. Besides, generative models are also applied to domain adaptation in \cite{Deng2017Nov,Wei_2018_CVPR}. Wang et al. \cite{Wang2018Mar} design a network learning an attribute-semantic and identity discriminative feature representation. Similarly, Li et al. \cite{Li2018Apr} leverages information across datasets and derives domain-invariant features via an adaptation and a re-ID network. 
Though all the above methods solve the adaptation problem, they are not supported by a theoretical framework and their generalization abilities are not verified in other re-ID tasks. 
Fan et al. \cite{Fan2017May} propose a progressive unsupervised learning method consisting of clustering and fine-tuning the network, which is similar to our self-training scheme. However, they only focuses on unsupervised learning, not unsupervised domain adaptation. In addition, their iteration framework is not guided by specific assumptions thus having no theoretical derived loss functions as ours.

\vspace{-5pt}
\section{Notations and Basic Definitions}\label{sec:not}
\vspace{-5pt}
In classification tasks, let $\matx \subseteq \fr^{d}$ be the input space and $\maty\subset\fr$ be the output space, and each sample from the input space is denoted by black lower case letters $\vecx\in\matx$. 
We denote source domain as $\doms$ and target domain as $\domt$, and both of them are probability distribution over the input space $\matx$.
Moreover, the real label of each sample is denoted by a labeling function $l:\matx\rightarrow\maty$. 
However, the above notations could not be directly used to analyze the re-ID tasks, because there is no same identity in two domains, i.e. $\doms$ and $\domt$ do not have the same output (label) space.
Fortunately, for re-ID tasks, by treating re-ID as classifying same or different data pairs we are still able to utilize the notations and former results with some simple reformulations.  

Specifically, in re-ID tasks, we have a training set consist of data pairs, which means that the input space is $(\matz,\matz)\subseteq\fr^{n\times n}$, and the output space is $\maty=\{0,1\}$, where $1$ means the identities in the pair are the same and $0$ means different. 
Observing that in re-ID tasks, the two domain indeed have some overlapping cues, such as color of clothes, wearing a backpack or not in person re-id. 
That is, we can encode the original data from the two domains with some feature variables or latent variables, and then it is reasonable to assume that distribution of features from two domains satisfy some criteria just as the assumptions used in \cite{BenDavid2014} for classification tasks. 
Formally, we denote the feature encoder as $\vecx(\cdot)$ and $\vecx: \matz\rightarrow\fr^d$, and then the labeling function is $l:\matx\times\matx\rightarrow\{0,1\}$, where $\matx\subseteq\fr^d$ is the extracted feature space.
For simplicity, we denote $l(\vecx(\vecz_1),\vecx(\vecz_2)) = l(\vecx_1,\vecx_2)$, where $\vecz_1, \vecz_2$ means two different raw data.
Note that the labeling function is symmetric, i.e. $l(\vecx_1,\vecx_2)=l(\vecx_2,\vecx_1)$. 

\vspace{-5pt}
\section{Assumptions and DA-Learnability}\label{sec:theo}
\vspace{-5pt}
In this section, we first introduce some assumptions reflecting how the source domain interacting with target domain. Then with these assumptions we show the learnability of unsupervised domain adaptive re-ID.

The first assumption is covariate shift, which means that the criteria of classifying data pairs are the same for source domain and target domain.
In other words, we have $l_\doms(\vecx)=l_\domt(\vecx)$ for classification tasks, and similarly we can define the covariate shift for re-id tasks on the extracted feature space. 

\begin{definition}[Covariate Shift]
We say that source and target distribution satisfy the covariate shift assumption if they have the same labeling function, i.e. if we have $l_\doms(\vecx_1,\vecx_2)=l_\domt(\vecx_1,\vecx_2)$. 
\end{definition} 
\vspace{-5pt}

Another assumption is inspired by the ``Probabilistic Lipschitzness'', which is originally proposed for semi-supervised learning in \cite{Urner2011} and then investigated with application to domain adaptation tasks in \cite{BenDavid2014}. This assumption captures the intuition that in a classification task, the data can be divided into label-homogeneous clusters and are separated by low-density regions. 
However, in re-id tasks, the labeling function is a multivariable function, thus the original Probabilistic Lipschitzness is not applicable. Note that the intuition of re-id tasks is that similar pairs can form as a cluster. That is, for an instance, the similar data can be divided into a cluster and the cluster is separated out from the data space with a low-density gap. Mathematically, we have the following definition.
\begin{definition}[Separately Probabilistic Lipschitzness (SPL)]
Let $\phi:\fr\rightarrow[0,1]$ be monotonically increasing. Symmetric function $f:\matx\times\matx\rightarrow\fr$ is $\phi$-SPL with respect to a distribution $\domd$ on $\matx$, if for all $\lambda>0$, 
\begin{equation}\label{eqn:spl}
\fp_{\vecx_1,\vecx_2\sim\domd} 
\left( \exists \vecy:|f(\vecx_1,\vecx_2) - f(\vecx_1,\vecy)| \geqslant \lambda \|\vecx_2-\vecy\| \right)
\leqslant \phi (\lambda)
\end{equation} 
\end{definition}
\vspace{-5pt}

To ensure the learnability of the domain adaptation task, we still need a critical assumption concerning how much overlap there is between the source and target domain. 
We again follow the assumption used in \cite{BenDavid2014} on the source and target distribution, which is a relaxation of the pointwise density ratio between the two distributions. 

\begin{definition}[Weight Ratio]
Let $\setb \subseteq 2^{\matx}$ be a collection of subsets of the input space $\matx$ measurable with respect to both $\doms$ and $\domt$. For some $\eta > 0$ we define the $\eta$-weight ratio of the source distribution and the target distribution with respect to $\setb$ as 
\vspace{-5pt}
\begin{equation}
C_{\setb,\eta}(\doms,\domt) = \inf_{\substack{b\in\setb \\ \domt(b)\geqslant\eta}} \frac{\doms(b)}{\domt(b)}
\end{equation}
Further, we define the weight ratio of the source distribution and the target distribution with respect to $\setb$ as 
\begin{equation}
	C_{\setb}(\doms,\domt) = \inf_{\substack{b\in\setb \\ \domt(b)\neq 0}} \frac{\doms(b)}{\domt(b)}
\end{equation}
\end{definition}
\vspace{-5pt}

Following the notations in \cite{BenDavid2014}, we also assume that our domain is the unit cube $\matx = [0, 1]^d$ and let $\setb$ denote the set of axis aligned rectangles in $[0, 1]^d$. For our re-ID tasks, the risk of a classifier $h$ on target domain is 
\begin{equation}
\risk_\domt(h) = \fe_{\vecx_1,\vecx_2\sim\domt} [\loss(h(\vecx_1,\vecx_2),l(\vecx_1,\vecx_2))].
\end{equation}
Let the Nearest Neighbor classifier be $\nnc$, then the following theorem implies the learnability of domain adaptive re-ID, of which the proof is included in supplemental materials.

\begin{theorem}
Let the domain be the unit cube,$\matx = [0, 1]^d$, and for some $C > 0$, let
$(\doms, \domt)$ be a pair of source and target distributions over $\matx$ satisfying
the covariate shift assumption, with $C_\setb(\doms,\domt)\geqslant C$, and their common deterministic
labeling function $l : \matx\times\matx \rightarrow \{0, 1\}$ satisfying the $\phi$-SPL property with
respect to the target distribution, for some function $\phi$. Then, for all $\epsilon, \delta > 0$, for all
$(\doms, \domt)$, if $S$ is a source generated sample set of size at least
\[
m \geqslant \frac{4}{\epsilon\delta Ce} \left( \phi^{-1}\left(\frac{\epsilon}{4}\right)\sqrt{d} \right)^d
\]
then, with probability at least $1 - \delta$ (over the choice of $S$), $\risk_\domt(\nnc)$ is at most $\epsilon$.
\end{theorem}

\vspace{-10pt}
\section{Reinforcing the Assumptions}\label{sec:min}
\vspace{-5pt}

In previous section, we show that with some assumptions on the extracted feature space, unsupervised domain adaptation is learnable. Thus we are concerned with how to train a feature extractor, i.e. encoder, satisfying the mentioned assumptions. Briefly speaking, we first derive several loss functions according to the assumptions and then iteratively train the encoder to minimize the loss functions via a self-training framework.

\paragraph{Self-training framework.} Assume that we have an encoder $\vecx$ and some samples $\domd$ with guessed label $l$ on target domain, and the loss function is $\loss(\vecx,\domd,l)$. In self-training, at first a $\vecx^{(i)}$ is used to extract features from all available unlabeled samples, and the target now is minimizing the loss through selecting samples, that is $\min_{\domd,l} \loss(\vecx^{(i)},\domd,l)$. 
On the next round, with these selected samples, the encoder $\vecx^{(i)}$ is updated by solving the minimization problem $\min_{\vecx} \loss(\vecx,\domd^{(i)},l^{(i)})$.

It is worthwhile to note that the covariate shift assumption only depends on the property of labeling function, thus in this section we only consider the proposed SPL and weight ratio. 

\subsection{Reinforcing the SPL}
Recall that the original data is $\vecz\in\matz$ and we wish to iteratively find a encoder $\vecx(\cdot)$ such that in the feature space the SPL property is satisfied as much as possible. So we first need a definition to evaluate whether one encoder is better than another concerning the SPL property. 

\begin{definition}\label{def:intramin}
Encoder $\vecx^a(\cdot)$ is said to be more clusterable than $\vecx^b(\cdot)$ with respect to a labeling function $l$ and a distribution $\domd$ over $\matz$, if there exists $\epsilon\in(0,1)$, and $\lambda\in\{\lambda_1,\lambda_2\}$ with $\lambda_1\lambda_2<0$, such that 
\vspace{-4pt}
\begin{equation*}
\begin{aligned}
\fp_{\vecz_1,\vecz_2\sim\domd} &\big( \exists\vecz_3:|l(\vecx^a(\vecz_1),\vecx^a(\vecz_2))-l(\vecx^a(\vecz_1),\vecx^a(\vecz_3))|-\epsilon\geqslant\lambda \|\vecx^a(\vecz_2)-\vecx^a(\vecz_3)\| \big) \\
\vspace{-5pt}
&\leqslant
\fp_{\vecz_1,\vecz_2\sim\domd} \big( \exists\vecz_3:|l(\vecx^b(\vecz_1),\vecx^b(\vecz_2))-l(\vecx^b(\vecz_1),\vecx^b(\vecz_3))|-\epsilon\geqslant\lambda \|\vecx^b(\vecz_2)-\vecx^b(\vecz_3)\| \big)
\end{aligned}
\end{equation*}
\end{definition}
\vspace{-5pt}
The above equation differs from the original SPL (\ref{eqn:spl}) for the reason that the original form is too strict to be satisfied.
Now we can easily define a loss function
\begin{equation*}
\loss(\vecx,\domd,l;\epsilon,\lambda) = \fe_{\vecz_1,\vecz_2\sim\domd} \big[  \exists\vecz_3:|l(\vecx(\vecz_1),\vecx(\vecz_2))-l(\vecx(\vecz_1),\vecx(\vecz_3))|-\epsilon\geqslant\lambda \|\vecx(\vecz_2)-\vecx(\vecz_3)\|  \big],
\end{equation*}
where $\domd$ means a set of samples and $l$ is the guessed labeling function.
However, directly performing optimization on the loss function is infeasible since the analytical form is unknown.
To overcome the difficulty, we adopt intra-cluster distance and inter-cluster distance,
\vspace{-5pt}
\begin{align}
&\lossintra(\vecx,\domd,l) = \sum_{l(\vecx(\vecz_1),\vecx(\vecz_2))=1} \|\vecx(\vecz_1)-\vecx(\vecz_2)\| , \\
&\lossinter(\vecx,\domd,l) = \sum_{l(\vecx(\vecz_1),\vecx(\vecz_2))=0} -\|\vecx(\vecz_1)-\vecx(\vecz_2)\| . 
\end{align}
\vspace{-8pt}

We show that minimizing $\lossintra$ and $\lossinter$ is appropriate for being more clusterable through the following theorem.
\begin{theorem}\label{theo:intramin}
For two encoders $\vecx^a, \vecx^b$, a distribution $\domd$ and a labeling function $l$, then
\vspace{-5pt}
\[
    \text{$\vecx^a$ is more clusterable than $\vecx^b$}
    \Leftrightarrow
    \left\{
    \begin{aligned}
    \lossintra(\vecx^a,\domd,l)\leqslant\lossintra(\vecx^b,\domd,l) \\
    \lossinter(\vecx^a,\domd,l)\leqslant\lossinter(\vecx^b,\domd,l)
    \end{aligned}
    \right.
\]
\end{theorem}
\vspace{-8pt}

For proof we refer reader to the supplemental materials.
Here, Definition \ref{def:intramin} and Theorem \ref{theo:intramin} describe how to evaluate an encoder with a fixed distribution $\domd$ and labeling function $l$. Obviously, we can fix the encoder and rewrite the results to evaluate the samples with guessed labels. For the sake of conciseness, the details are omitted. 
When $\domd$ and $l$ are fixed during the iteration procedure, minimizing $\lossintra$ and $\lossinter$ are straightforward.
Contrastingly, we have to focus more on the strategy of picking out samples with guessed labels. 

\vspace{-5pt}
\paragraph{Selecting samples via clustering.}
In spite of the similarity between $\lossintra$ and $\lossinter$, they do not share a same strategy regarding the sample selection step. 
For $\lossintra$, if all the data in $\domt$ are encoded with a $\vecx$, then for each pair $(\vecx_i,\vecx_j)$, it is natural to assume that a smaller $\|\vecx_i-\vecx_j\|$ implies a higher confidence that $l(\vecx_i,\vecx_j)=1$. Likewise, a larger $\|\vecx_i-\vecx_j\|$ implies a higher confidence that $l(\vecx_i,\vecx_j)=0$. But choosing a high confidence different pair as training data does not really improve the real performance, because the accuracy is more sensitive about the minimal distance of different pairs, i.e., $\inf_{i:l(\vecx_d,\vecx_i)=0}\|\vecx(\vecz_d)-\vecx(\vecz_i)\|$. So rather than directly selecting different pairs, we treat the selected samples as a series of clusters and dissimilar pairs are selected on the basis of different clusters. 
That is to say, in order to minimize $\lossintra$ and $\lossinter$ simultaneously, we perform clustering on the data with guessed labels. 

\vspace{-5pt}
\paragraph{Distance metrics and loss functions.}
Up to this point, we are facing an unsupervised clustering problem, which is largely settled by the distance metric. In other words, designing a sample selecting strategy to minimize a loss turns into designing a distance metric between samples, and a better distance should lead to a lower $\lossintra$ and $\lossinter$. It is a common practice in image retrieval that the contextual similarity \cite{Jegou2007} measure is more robust and beneficial for a lower $\lossintra$.

In our practice, we adopt the $k$-reciprocal encoding in \cite{Zhong2017Jan} as the distance metric, which is a variation of Jaccard distance between nearest neighbors sets.
Precisely, with an encoder $\vecx$, all samples from $\domt$ are encoded and with these features a distance matrix $M\in\fr^{m_t\times m_t}$ is computed where $M_{ij}=\|\vecx_i-\vecx_j\|^2$ and $m_t$ is the total number of target samples. Then $M$ is updated by 
\vspace{-5pt}
\begin{equation}
	M_{ij} = \left\{
	\begin{aligned}
		  & e^{-M_{ij}} & \text{ if $j\in\seti_i$}, \\
		  & 0           & \text{otherwise}.
	\end{aligned}
	\right.
\end{equation}
\vspace{-10pt}

where the indices set $\seti_i$ is the so called robust set for $\vecx_i$. $\seti_i$ is determined by first choosing mutual $k$ nearest neighbors for the probe, then incrementally adding elements. Specifically, denote the indices of mutual $k$ nearest neighbors of $\vecx_i$ as $\setk_k(\vecx_i)$ and then for all $s\in\setk_k(\vecx_i)$, if $|\setk_k(\vecx_i)\cap\setk_{k\over{2}}(\vecx_s)|\geqslant \frac{2}{3}|\setk_{k\over{2}}(\vecx_s)|$, let $\seti_i\leftarrow\setk_k(\vecx_i)\cup\setk_{k\over{2}}(\vecx_s)$.
In particular, for a pair $(\vecx_i,\vecx_j)$, we have
\begin{equation}
	\disj(\vecx_i,\vecx_j) = 1- \frac{\sum_{k=1}^{m_t} \min(M_{ik},M_{jk})}{\sum_{k=1}^{m_t} \max(M_{ik},M_{jk})}.
\end{equation}

\subsection{Reinforcing the weight ratio}
As mentioned before, weight ratio is a crucial part to support the learnability of domain adaptation. 
Apart from directly define a loss based on the original weight ratio definition, a similar way as the SPL case is minimizing the loss
\begin{equation}
\lossfunwr(\vecx,\domd) = \fe_{\vecz_d\sim\domd} \left[ \inf_{\vecz_s\sim\doms}\|\vecx(\vecz_d)-\vecx(\vecz_s)\| \right],
\end{equation}
where $\doms$ is the source domain. 
The intuition here is to enhance the degree of similarity, which means that each target feature is close to some source features. 
We denote $C_{\setb,\eta}(\doms,\domt;\vecx)$ as the weight ratio when using $\vecx$ as the encoder, where $\setb$ is defined in Section \ref{sec:theo}. The following theorem demonstrate that our $\lossfunwr$ makes sense and the proof is in the supplemental materials.
\begin{theorem}
For two encoders $\vecx^a, \vecx^b$, a distribution $\domd$, if $\eta$ is a random variable and its support is a subset of $\fr^+$ , then
\[
\lossfunwr(\vecx^a,\domd)\leqslant\lossfunwr(\vecx^b,\domd) \Leftrightarrow \fe\left[C_{\setb,\eta}(\doms,\domd;\vecx^a)\right]\geqslant \fe\left[C_{\setb,\eta}(\doms,\domd;\vecx^b)\right]
\]
\end{theorem}

However, unlike $\lossinter$ and $\lossintra$, it is hard to optimize on $\vecx$ for $\lossfunwr$ because of the infimum. On the other hand, selecting samples is easily done via giving more confidence to the sample with smaller $\inf_{\vecz_s\sim\doms}\|\vecx(\vecz_d)-\vecx(\vecz_s)\|$.
More specifically, for each $\vecx_i$ from $\domt$, we search the nearest neighbor in $\doms$. The function measuring the confidence for each $\vecx_i$ is denoted by
\begin{equation}
	\diswr(\vecx_i) = 1 - e^{-\|\vecx_i-N_\doms(\vecx_i)\|^2}.
\end{equation}
where $N_\doms(\vecx_i)$ means the nearest neighbor of $\vecx_i$ in source domain $\doms$, and a smaller $\diswr$ means a higher confidence. To transform $\diswr$ and $\disj$ onto the same scale, we perform a simple normalization on $\diswr$, i.e., divided by $\max_i \diswr(\vecx_i)$.
Combining with $\disj$, the final distance matrix is $M_{ij}=d(\vecx_i,\vecx_j)$ and
\begin{equation}\label{eq:confidence}
	d(\vecx_i,\vecx_j) = (1-\lambda) \disj(\vecx_i,\vecx_j) + \lambda (\diswr(\vecx_i)+\diswr(\vecx_j)),
\end{equation}
where $\lambda\in[0,1]$ is a balancing parameter. 

\subsection{Overall algorithm}
So far, general outlines of reinforcing the assumptions have been elaborated, except the details about the clustering method. 
In our framework, a good clustering method should possess the following properties: (a) it does not require the number of clusters as an input, because in fact a cluster means an identity and the number of identities is trivial and unknown; (b) it is able to avoid pairs of low confidence, that is allowing some points not belonging to any clusters; (c) it is scalable enough to incorporate our theoretically derived distance metric. We employ the clustering method named DBSCAN \cite{Ester}, which has stood the test of time and exactly have the mentioned advantages. 

Now we provide some other practical details of our domain adaptive re-ID algorithm.
At the beginning, an encoder $\vecx^{(0)}$ is well trained on $\doms$ and all the pairs are computed with Eqn.(\ref{eq:confidence}).
Next, we describe how we set the threshold controlling whether a pair should be used to train. Intuitively, the threshold should be irrelevant to tasks since the scale of $d$ varies from tasks. 
So in our method, we first sort all the distance from lowest to highest and the average value of top $pN$ pairs is set to be the threshold $\tau$, where  $N$ is the total number of possible pairs and $p$ is percentage. 
On these data with pseudo-labels, the encoder is then trained with triplet loss \cite{Weinberger2009}. 
Our whole framework is concluded in Algorithm \ref{algo}.

\begin{algorithm}
	\begin{footnotesize}
		\SetKwInOut{Input}{input}\SetKwInOut{Output}{output}
		\Input{source domain dataset $S$, unlabeled target domain dataset $T$ with $m_t$ samples, balancing parameter $\lambda$, percentage $p$, the minimum size of a cluster $N_1$, iteration number $N_2$}
		\Output{an encoder $\vecx$ for target domain}
		Train an encoder $\vecx^{(0)}$ on $S$\;
		Compute $T^{(0)}=\vecx^{(0)}(T), S^{(0)}=\vecx^{(0)}(S)$\;
		Compute a distance matrix $M^{(0)}$ on $T^{(0)}, S^{(0)}$ by Eqn.(\ref{eq:confidence})\;
		Sort all the $N$ elements in $M^{(0)}$ from low to high and record the mean of top $pN$ values as threshold $\tau$\;
		Select train data $D^{(0)}=\texttt{DBSCAN}(M^{(0)};\tau,N_1)$\;
		Train $\vecx^{(1)}$ on $D$\;
		\For{$i=1$ \KwTo $N_2$}{
		Compute $T^{(i)}=\vecx^{(i)}(T^{(i-1)}), S^{(i)}=\vecx^{(i)}(S^{(i-1)})$\;
		Compute $M^{(i)}$ on $T^{(i)}, S^{(i)}$\;
		Select $D^{(i)}=\texttt{DBSCAN}(M^{(i)};\tau,N_1)$\;
		Train $\vecx^{(i+1)}$ on $D^{(i)}$\;
		}
		\caption{Unsupervised Domain Adaptation for Re-ID}\label{algo}
	\end{footnotesize}
\end{algorithm}
\vspace{-5pt}

\section{Experiments}\label{sec:exp}
In this section, we test our unsupervised domain adaptation algorithm on person re-ID and vehicle re-ID. The performance are evaluated by cumulative matching characteristic (CMC) and mean Average Precision (mAP), which are multi-gallery-shot evaluation metrics defined in \cite{Zheng2015Dec}.

\paragraph{Parameter settings and implementation details.}
In all the following re-ID experiments, we empirically set $\lambda=0.1,p=1.6\times10^{-3}$, $N_1=4$ and $N_2=20$. Basically, the encoder is ResNet-50 \cite{he2016deep} pre-trained on ImageNet.
Both triplet and softmax loss are used for initializing the network on source domain, while only triplet loss is used for refining the encoder on target domain. More details about the network, training parameters and visual examples from different domains are included in the supplemental materials. Moreover, in the supplemental materials we also investigate other distance metrics and clustering methods.

\subsection{Person re-ID}
\begin{wraptable}{r}{0.5\textwidth}
	\begin{footnotesize}
		\caption{The details of datasets used in our experiments.}
		\label{ta:dataset}
		\begin{tabular}{c|cc|cc}
			\hline\hline
			\multirow{2}{*}{Datasets} & \multicolumn{2}{c|}{Training} & \multicolumn{2}{c}{Testing}   \\ \cline{2-5}
			                           & \#IDs & \#Images & \#IDs & \#Images \\ \hline
			Market \cite{Zheng2015Dec} & 751   & 12,936   & 750   & 19,732   \\
			Duke \cite{Ristani2016}    & 702   & 16,522   & 702   & 19,889   \\
			VeRi \cite{Liu7553002}     & 576   & 37,778   & 200   & 13,257   \\
			PKU \cite{liu2016deep}     & 2,290 & 24,157   & -     & -        \\
			\hline\hline
		\end{tabular}
	\end{footnotesize}
\end{wraptable}
Market-1501 \cite{Zheng2015Dec} and DukeMTMC-reID \cite{Ristani2016} are two large scale datasets and frequently used for unsupervised domain adaptation experiments. Both of the two datasets are split into a training set and a testing set. The details including the number of identities and images are shown in Table \ref{ta:dataset}.

Comparison methods are selected in three aspects. Firstly, we show the performance of direct transfer, that is directly using the initial source-trained encoder on the target domain. Also, the plain self-training scheme is compared as a baseline, which means sample selection only depends on their Euclidean distance.
Secondly, our method is compared with three most recent state-of-the-art methods\footnote{Our results also outperforms PTGAN\cite{Wei_2018_CVPR} by large margin, but the comparison with PTGAN is not shown here since we adopt a different backbone network.}: SPGAN \cite{Deng2017Nov}, TJ-AIDL \cite{Wang2018Mar} and ARN \cite{Li2018Apr}. We report the original results quoted from in their papers.
Thirdly, we show the results of our methods with and without $\diswr$, which can be viewed as ablation studies.
The results are shown in Table \ref{ta:person}, from which we can observe the following facts: (a) The accuracy of self-training baseline is high and even better than two recent methods, indicating that our clustering based self-training scheme is fairly good;
(b) The version without $\diswr$ is better than self-training baseline, which shows the effectiveness of $\disj$, and after incorporated with $\diswr$ the final method achieves the highest accuracy, reflecting the advantage of $\diswr$. Thus our two assumptions are both useful according to the ablation studies.
(c) Although the proposed $\diswr$ is beneficial, the increase of accuracy brought by it varies from different tasks. We think this is related to the distribution of source and target domains. Please refer to for more discussion in \ref{para:diswr} on this problem.

Furthermore, we draw the mAP curves (Figure \ref{fig:conver}) during the iterations of the adaptation task Duke$\rightarrow$Market, in which self-training baseline, using distance without $\diswr$ and $\lambda=\{0.05,0.1,0.5,0.7\}$ are compared. We can see that except the baseline, all the curves have a similar tendency toward convergence. A subtle distinction is that after 18 iterations methods with smaller $\lambda$ become unstable, while methods with larger $\lambda$ move toward convergence.

\begin{table}[]
	\centering
	\caption{Comparison of unsupervised domain adaptive person re-ID methods.}
	\label{ta:person}
	\footnotesize
	\begin{tabular}{c|cccc|cccc}
		\hline\hline
		\multirow{2}{*}{Methods} & \multicolumn{4}{c|}{DukeMTMC-reID$\rightarrow$Market-1501} & \multicolumn{4}{c}{Market-1501$\rightarrow$DukeMTMC-reID}   \\ \cline{2-9}
		                           & rank-1    & rank-5    & rank-10   & mAP       & rank-1    & rank-5    & rank-10   & mAP       \\ \hline\hline
		Direct Transfer            & 46.8      & 64.6      & 71.5      & 19.1      & 27.3      & 41.2      & 47.1      & 11.9      \\
		Self-training Baseline     & 66.7      & 80.0      & 85.0      & 39.6      & 40.8      & 53.9      & 60.5      & 24.7      \\
		\hline
		SPGAN \cite{Deng2017Nov}   & 57.7      & 75.8      & 82.4      & 26.7      & 46.4      & 62.3      & 68.0      & 26.2      \\
		TJ-AIDL \cite{Wang2018Mar} & 58.2      & 74.8      & 81.1      & 26.5      & 44.3      & 59.6      & 65.0      & 23.0      \\
		ARN \cite{Li2018Apr}       & 70.3      & 80.4      & 86.3      & 39.4      & 60.2      & 73.9      & 79.5      & 33.4      \\
		\hline
		\emph{Ours} w/o $\diswr$   & 75.1      & 88.7      & 92.4      & 52.5      & 68.1      & \bf{80.1} & 83.2      & \bf{49.0} \\
		\emph{Ours}                & \bf{75.8} & \bf{89.5} & \bf{93.2} & \bf{53.7} & \bf{68.4} & \bf{80.1} & \bf{83.5} & \bf{49.0} \\
		\hline\hline
	\end{tabular}
\end{table}

\subsection{Vehicle re-ID}

We use VeRi-776 \cite{Liu7553002} and part of PKU-VehicleID \cite{liu2016deep} for vehicle re-ID experiments\footnote{In PKU-VehicleID, the camera information is not provided but needed when computing the CMC and mAP, so we only test with the setting that PKU-VehicleID as source dataset and VeRi-776 as target dataset.}, the details are included in Table \ref{ta:dataset}.
Unlike person re-ID, currently there are no unsupervised domain adaptation algorithms designed for vehicle re-ID. Thus, we use the existing solutions for person re-ID as comparisons\footnote{We only test SPGAN. Because (1) source code of ARN is not available; (2) TJ-AIDL requires attribute labels as an input, which is not available in vehicle re-ID datasets. For SPGAN, the experiments are carried out with their default parameters for person re-ID.}.
As shown in Table \ref{ta:vehicle}, not only are the conclusions from person re-ID verified again, but also the generalization ability of our method is shown.
We discover that the compared SPGAN generates quite presentable images and we put the images into supplemental materials, but their accuracy is still lower than the self-training baseline, not to mention our proposed method.

\begin{figure}
	\begin{minipage}[t]{0.54\textwidth}
		\small
		\centering
		\vspace{0pt}
		\captionof{table}{Comparison of unsupervised domain adaptive vehicle re-ID methods.}
		\label{ta:vehicle}
		\begin{tabular}{c|cccc}
			\hline\hline
			\multirow{2}{*}{Methods} & \multicolumn{4}{c}{PKU-VehicleID$\rightarrow$VeRi-776}  \\ \cline{2-5}
			                         & rank-1    & rank-5    & rank-10   & mAP       \\ \hline\hline
			Direct Transfer          & 52.1      & 65.1      & 71.1      & 14.6      \\
			\begin{tabular}{@{}c@{}}Self-training \\ Baseline\end{tabular}
			                         & 74.4      & 81.6      & 84.6      & 33.5      \\
			\hline
			SPGAN \cite{Deng2017Nov} & 57.4      & 70.0      & 75.6      & 16.4      \\
			\hline
			\emph{Ours} w/o $\diswr$ & 76.7      & 85.5      & \bf{89.3} & 35.3      \\
			\emph{Ours}              & \bf{76.9} & \bf{85.8} & 89.0      & \bf{35.8} \\
			\hline\hline
		\end{tabular}
	\end{minipage}
	\hfill
	\begin{minipage}[t]{0.4\textwidth}
		\centering
		\vspace{-10pt}
		\includegraphics[width=\textwidth]{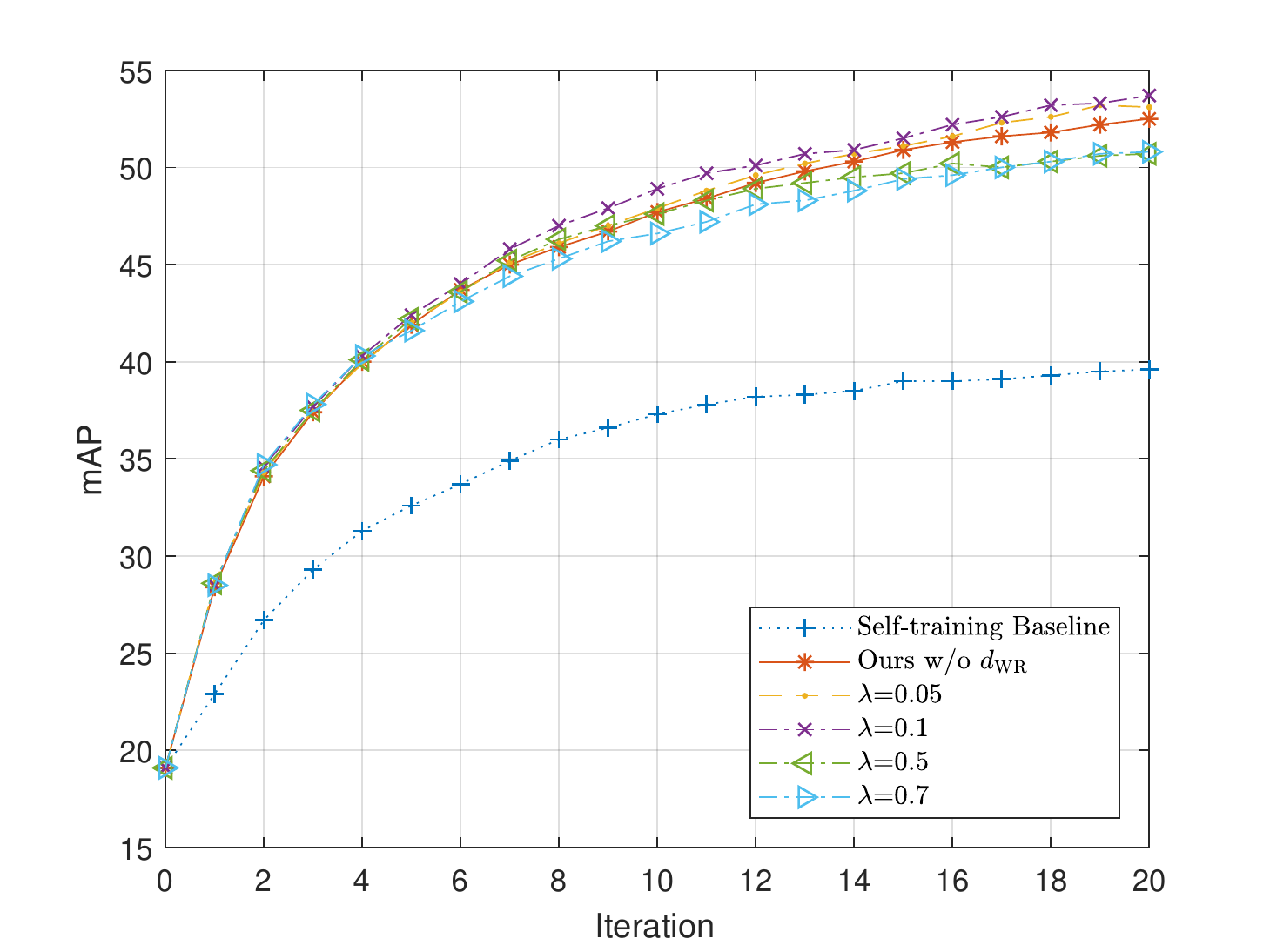}
		\caption{Convergence comparison}
		\label{fig:conver}
	\end{minipage}
	\vspace{-15pt}
\end{figure}

\section{Conclusion and Future Work}
In this work, we bridge the gap between theories of unsupervised domain adaptation and re-id tasks. Inspired by previous work \cite{BenDavid2014}, we make assumptions on the extracted feature space and then show the learnability of unsupervised domain adaptive re-id tasks. Treating the assumptions as the goal of our encoder, several loss functions are proposed and then minimized via self-training framework. 

Though the proposed solution is effective and outperforms state-of-the-art methods, there are still problems unsolved in our algorithm. Firstly, with regard of the weight ratio assumption, we propose the loss function $\lossfunwr$, which is ignored when updating the encoder because of the intractable infimum. So designing another feasible loss function is an interesting direction of research. Another promising issue is to improve the data selecting step in the self-training scheme. We turn the data selecting step into a clustering problem, which can be thought of as a version with hard threshold. This suggest that there may be a better strategy which utilize the relative values between distances. 
We hope that our analyses could open the door to develop new domain adaptive re-ID tasks and can lift the burden of designing large and complicate networks.

\begin{small}
\bibliographystyle{plain}
\bibliography{da}
\end{small}
\newpage 

\appendix
\addcontentsline{toc}{section}{Supplementary Materia}
\section*{Supplementary Materia}
\section{Theorems and Proofs}
To prove Theorem 1, we first give a lemma on the upper error bound of $\risk_\domt(\nnc)$. 
Let $\setb$ denote the set of axis aligned rectangles in $[0, 1]^d$ and, given some $\eta > 0$, let $\setb_\eta$ denotes the class of axis aligned rectangles with side-length $\eta$. For a sample set $S$ from source domain, we have
\begin{lemma}\label{theo:bound}
	Let the domain be the unit cube, $\matx = [0, 1]^d$, and for some $C > 0$ and
	some $\eta \geqslant 0$, let $(\doms, \domt)$ be source and target distributions over
	$\matx$ satisfying the covariate shift assumption, with $C_{\setb,\eta} (\doms, \domt) \geqslant C$, and their
	common re-id labeling function $l : \matx\times\matx \rightarrow \{0, 1\}$ satisfying the $\phi$-SPL property with respect to the target distribution, for some function $\phi$. Then, for all $m$, and all $(\doms, \domt)$,
	\begin{equation}
		\fe_{S\sim\doms^m} [\risk_\domt(\nnc)] \leqslant 2\phi(\frac{1}{\eta\sqrt{d}}) + \frac{2}{\eta^dCme}
	\end{equation}
\end{lemma}
\begin{proof}
	A test pair $(\vecx_1,\vecx_2)$ gets the wrong label under two conditions: (a) at least one test data do not have a close neighbor with all the $m$ training data; (b) $(\vecx_1,\vecx_2)$ have a close neighbor pair which have the opposite label. For (a), we can use the results from Lemma 7 and  Theorem 8 in \cite{BenDavid2014}. Specifically, let $C_1,C_2, \cdots ,C_1/\eta^d$ be a cover of the set $[0, 1]^d$ using boxes of side-length $\eta$. We have
	\begin{equation}
		\fe_{S\sim\doms^m} \left[ \sum_{i:S\cap C_i=\varnothing} \domt(C_i) \right] \leqslant \frac{1}{\eta^dCme}.
	\end{equation}
	If $\vecx$ is in the box $C_{\vecx}$, then the probability of (a) can be expressed as $\fp(C_{\vecx_1}\cap S=\varnothing\lor C_{\vecx_2}\cap S=\varnothing)$. Observing that
	\[
		\fp(C_{\vecx_1}\cap S=\varnothing\lor C_{\vecx_2}\cap S=\varnothing) \leqslant \fp(C_{\vecx_1}\cap S=\varnothing) + \fp(C_{\vecx_2}\cap S=\varnothing)
	\]
	and $\fp(C_{\vecx}\cap S=\varnothing) = \sum_{i:S\cap C_i=\varnothing} \fp(C_i)$, so (a) is bounded by $\frac{2}{\eta^dCme}$.
	For (b), we denote the nearest neighbor to $\vecx$ in $S$ is $N_S(\vecx)$ and then (b) means in the box we have
	\begin{equation}
		l(\vecx_1,\vecx_2)\neq l(N_S(\vecx_1),N_S(\vecx_2)) \wedge \|N_S(\vecx_1)-\vecx_1\|\leqslant\eta\sqrt{d} \wedge \|N_S(\vecx_2)-\vecx_2\|\leqslant\eta\sqrt{d}.
	\end{equation}
	Seeing that
	\begin{align*}
		|l(\vecx_1,\vecx_2) - l & (N_S(\vecx_1),N_S(\vecx_2))|                                                                                    \\
		                        & = |l(\vecx_1,\vecx_2) -l(\vecx_1,N_S(\vecx_2))+l(\vecx_1,N_S(\vecx_2)) - l(N_S(\vecx_1),N_S(\vecx_2))|          \\
		                        & \leqslant|l(\vecx_1,\vecx_2) -l(\vecx_1,N_S(\vecx_2))|+|l(\vecx_1,N_S(\vecx_2)) - l(N_S(\vecx_1),N_S(\vecx_2))| \\
	\end{align*}
	So
	\begin{align*}
		\fp \Big(l(\vecx_1,\vecx_2)\neq & l(N_S(\vecx_1),N_S(\vecx_2)) \wedge \|N_S(\vecx_1)-\vecx_1\|\leqslant\eta\sqrt{d} \wedge \|N_S(\vecx_2)-\vecx_2\|\leqslant\eta\sqrt{d} \Big) \\
		                                & \leqslant \fp\left(|l(\vecx_1,\vecx_2) -l(\vecx_1,N_S(\vecx_2))|\geqslant \frac{1}{\eta\sqrt{d}}\|N_S(\vecx_2)-\vecx_2\|\right)              \\
		                                & \qquad+\fp\left(|l(\vecx_1,N_S(\vecx_2)) - l(N_S(\vecx_1),N_S(\vecx_2))|\geqslant \frac{1}{\eta\sqrt{d}}\|N_S(\vecx_1)-\vecx_1\|\right)      \\
		                                & \leqslant 2\phi(\frac{1}{\eta\sqrt{d}})
	\end{align*}
	Combining the two bounds together, we conclude our proof.
\end{proof}

If we have a stronger weight ratio assumption, i.e. $C_\setb(\doms,\domt)\geqslant C$, we get the following result of domain adaptation learnability.

\begin{theorem}
	Let the domain be the unit cube,$\matx = [0, 1]^d$, and for some $C > 0$, let
	$(\doms, \domt)$ be a pair of source and target distributions over $\matx$ satisfying
	the covariate shift assumption, with $C_\setb(\doms,\domt)\geqslant C$, and their common deterministic
	labeling function $l : \matx\times\matx \rightarrow \{0, 1\}$ satisfying the $\phi$-SPL property with
	respect to the target distribution, for some function $\phi$. Then, for all $\epsilon, \delta > 0$, for all
	$(\doms, \domt)$, if $S$ is a source generated sample set of size at least
	\[
		m \geqslant \frac{4}{\epsilon\delta Ce} \left( \phi^{-1}\left(\frac{\epsilon}{4}\right)\sqrt{d} \right)^d
	\]
	then, with probability at least $1 - \delta$ (over the choice of $S$), the target error of the Nearest
	Neighbor classifier is at most $\epsilon$.
\end{theorem}
\begin{proof}
	From the proof in Theorem \ref{theo:bound}, the error was bounded under two circumstances. As for (a), we apply Markov's inequality and get
	\begin{equation}
		\fe_{S\sim\doms^m} \left[ 2 \sum_{i:S\cap C_i=\varnothing} \domt(C_i) \geqslant \frac{\epsilon}{2} \right] \leqslant \frac{4}{\epsilon\eta^dCme}
	\end{equation}
	Then for (b), we just set $2\phi(\frac{1}{\eta\sqrt{d}})=\frac{\epsilon}{2}$, so $\eta=\frac{\sqrt{d}}{\phi^{-1}(\epsilon/4)}$. Finally, setting the probability to be smaller than $\delta$ yields that if
	\[
		m \geqslant \frac{4}{\epsilon\delta Ce} \left( \phi^{-1}\left(\frac{\epsilon}{4}\right)\sqrt{d} \right)^d
	\]
	then with probability at least $1 - \delta$, the target error of the Nearest Neighbor classifier is at most $\epsilon$.
\end{proof}

\begin{theorem}
	For two encoders $\vecx^a, \vecx^b$, a distribution $\domd$ and a labeling function $l$, then
	\[
		\text{$\vecx^a$ is more clusterable than $\vecx^b$}
		\Leftrightarrow
		\left\{
		\begin{aligned}
			\lossintra(\vecx^a,\domd,l)\leqslant\lossintra(\vecx^b,\domd,l) \\
			\lossinter(\vecx^a,\domd,l)\leqslant\lossinter(\vecx^b,\domd,l)
		\end{aligned}
		\right.
	\]
\end{theorem}
\begin{proof}
	$(\Rightarrow)$ There exists $\epsilon\in(0,1)$, and $\lambda\in\{\lambda_1,\lambda_2\}$ with $\lambda_1\lambda_2<0$, such that
	\begin{align*}
		\fp_{\vecz_1,\vecz_2\sim\domd} & \big( \exists\vecz_3:|l(\vecx^a(\vecz_1),\vecx^a(\vecz_2))-l(\vecx^a(\vecz_1),\vecx^a(\vecz_3))|-\epsilon\geqslant\lambda \|\vecx^a(\vecz_2)-\vecx^a(\vecz_3)\| \big) \\
		                               & \leqslant
		\fp_{\vecz_1,\vecz_2\sim\domd} \big( \exists\vecz_3:|l(\vecx^b(\vecz_1),\vecx^b(\vecz_2))-l(\vecx^b(\vecz_1),\vecx^b(\vecz_3))|-\epsilon\geqslant\lambda \|\vecx^b(\vecz_2)-\vecx^b(\vecz_3)\| \big)
	\end{align*}
	When $l(\vecz_1,\vecz_2)=1,l(\vecz_1,\vecz_3)=1$, and $\lambda=\lambda_1<0$,
	\begin{equation*}
		\fp_{\vecz_1,\vecz_2\sim\domd} \big( \exists\vecz_3: \|\vecx^a(\vecz_2)-\vecx^a(\vecz_3)\|\geqslant-{\epsilon\over{\lambda_1}} \big)
		\leqslant
		\fp_{\vecz_1,\vecz_2\sim\domd} \big( \exists\vecz_3: \|\vecx^b(\vecz_2)-\vecx^b(\vecz_3)\|\geqslant-{\epsilon\over{\lambda_1}} \big).
	\end{equation*}
	So  $\lossintra(\vecx^a,\domd,l)\leqslant\lossintra(\vecx^b,\domd,l).$
	And let $l(\vecz_1,\vecz_2)=1,l(\vecz_1,\vecz_3)=0$, and $\lambda=\lambda_2>0$, then
	\[
		\fp_{\vecz_1,\vecz_2\sim\domd} \big( \exists\vecz_3: \|\vecx^a(\vecz_2)-\vecx^a(\vecz_3)\|\leqslant{1-\epsilon\over{\lambda_1}} \big)
		\leqslant
		\fp_{\vecz_1,\vecz_2\sim\domd} \big( \exists\vecz_3: \|\vecx^b(\vecz_2)-\vecx^b(\vecz_3)\|\leqslant{1-\epsilon\over{\lambda_1}} \big).
	\]
	So $\lossinter(\vecx^a,\domd,l)\leqslant\lossinter(\vecx^b,\domd,l).$

	$(\Leftarrow)$ We have
	\[
		\sum_{l(\vecx^a(\vecz_1),\vecx^a(\vecz_2))=1} \|\vecx^a(\vecz_1)-\vecx^a(\vecz_2)\| \leqslant \sum_{l(\vecx^b(\vecz_1),\vecx^b(\vecz_2))=1} \|\vecx^b(\vecz_1)-\vecx^b(\vecz_2)\|.
	\]
	Denote $C_1$ as the mean value of $\lossintra(\vecx^b,\domd,l)$, then
	\[
		\fp_{\vecz_1,\vecz_2\sim\domd} \big( \exists\vecz_3: \|\vecx^a(\vecz_2)-\vecx^a(\vecz_3)\|\geqslant C_1 \big)
		\leqslant
		\fp_{\vecz_1,\vecz_2\sim\domd} \big( \exists\vecz_3: \|\vecx^b(\vecz_2)-\vecx^b(\vecz_3)\|\geqslant C_1 \big).
	\]
	In like manner, denote $C_2$ as the mean value of $\lossinter(\vecx^a,\domd,l)$, then
	\[
		\fp_{\vecz_1,\vecz_2\sim\domd} \big( \exists\vecz_3: \|\vecx^a(\vecz_2)-\vecx^a(\vecz_3)\|\leqslant C_2 \big)
		\leqslant
		\fp_{\vecz_1,\vecz_2\sim\domd} \big( \exists\vecz_3: \|\vecx^b(\vecz_2)-\vecx^b(\vecz_3)\|\leqslant C_2 \big).
	\]
\end{proof}

\begin{theorem}
	For two encoders $\vecx^a, \vecx^b$, a distribution $\domd$, if $\eta$ is a random variable and its support is a subset of $\fr^+$ , then
	\[
		\lossfunwr(\vecx^a,\domd)\leqslant\lossfunwr(\vecx^b,\domd) \Leftrightarrow \fe\left[C_{\setb,\eta}(\doms,\domd;\vecx^a)\right]\geqslant \fe\left[C_{\setb,\eta}(\doms,\domd;\vecx^b)\right]
	\]
\end{theorem}
\begin{proof}
	\begin{align*}
		                     & \lossfunwr(\vecx^a,\domd)\leqslant\lossfunwr(\vecx^b,\domd) \\
		\Leftrightarrow\quad &
		\fp_{ \substack{\vecz_d\sim\domd \\ \vecz_s\sim\doms} } (\|\vecx^a(\vecz_d)-\vecx^a(\vecz_s)\|\leqslant\eta) \geqslant
		\fp_{ \substack{\vecz_d\sim\domd \\ \vecz_s\sim\doms} } (\|\vecx^b(\vecz_d)-\vecx^b(\vecz_s)\|\leqslant\eta) \\
		\Leftrightarrow\quad & (\forall b_0\in\setb)
		\fp_{ \substack{\vecz_d\sim\domd \\ \vecz_s\sim\doms} } (\vecx^a(\vecz_s)\in b_0| \vecx^a(\vecz_d)\in b_0) \geqslant
		\fp_{ \substack{\vecz_d\sim\domd \\ \vecz_s\sim\doms} } (\vecx^b(\vecz_s)\in b_0| \vecx^b(\vecz_d)\in b_0) \\
		\Leftrightarrow\quad &
		\fp_{ \vecz_s\sim\doms } (\vecx^a(\vecz_s)\in b_0| \domt(b_0;\vecx^a)\geqslant\eta ) \geqslant
		\fp_{ \vecz_s\sim\doms } (\vecx^b(\vecz_s)\in b_0| \domt(b_0;\vecx^b)\geqslant\eta ) \\
		\Leftrightarrow\quad &
		\fe\left[C_{\setb,\eta}(\doms,\domd;\vecx^a)\right]\geqslant \fe\left[C_{\setb,\eta}(\doms,\domd;\vecx^b)\right]
	\end{align*}
\end{proof}
\section{Additional Experimental Details and Results}

We present the structure of the paper in Figure \ref{fig:structure} and the most important contributions in our work are Theorem 2 and 3, both of which aim to turn the abstract and somewhat too theoretical assumptions into practical loss functions. Although Theorem 1 seems like a straightforward extension of previous work \cite{BenDavid2014}, it plays a fundamental role in the paper. Through the DA-learnability shown in Theorem 1, we can see that the three assumptions imposed on the distribution of two domains in Section 3 are sufficient for solving the domain adaptive re-ID problem. In other words, the sufficiency of reinforcing the three assumptions in Section 4 is shown via Theorem 1.

\begin{figure}[h]
	\centering
	\includegraphics[width=0.45\textwidth]{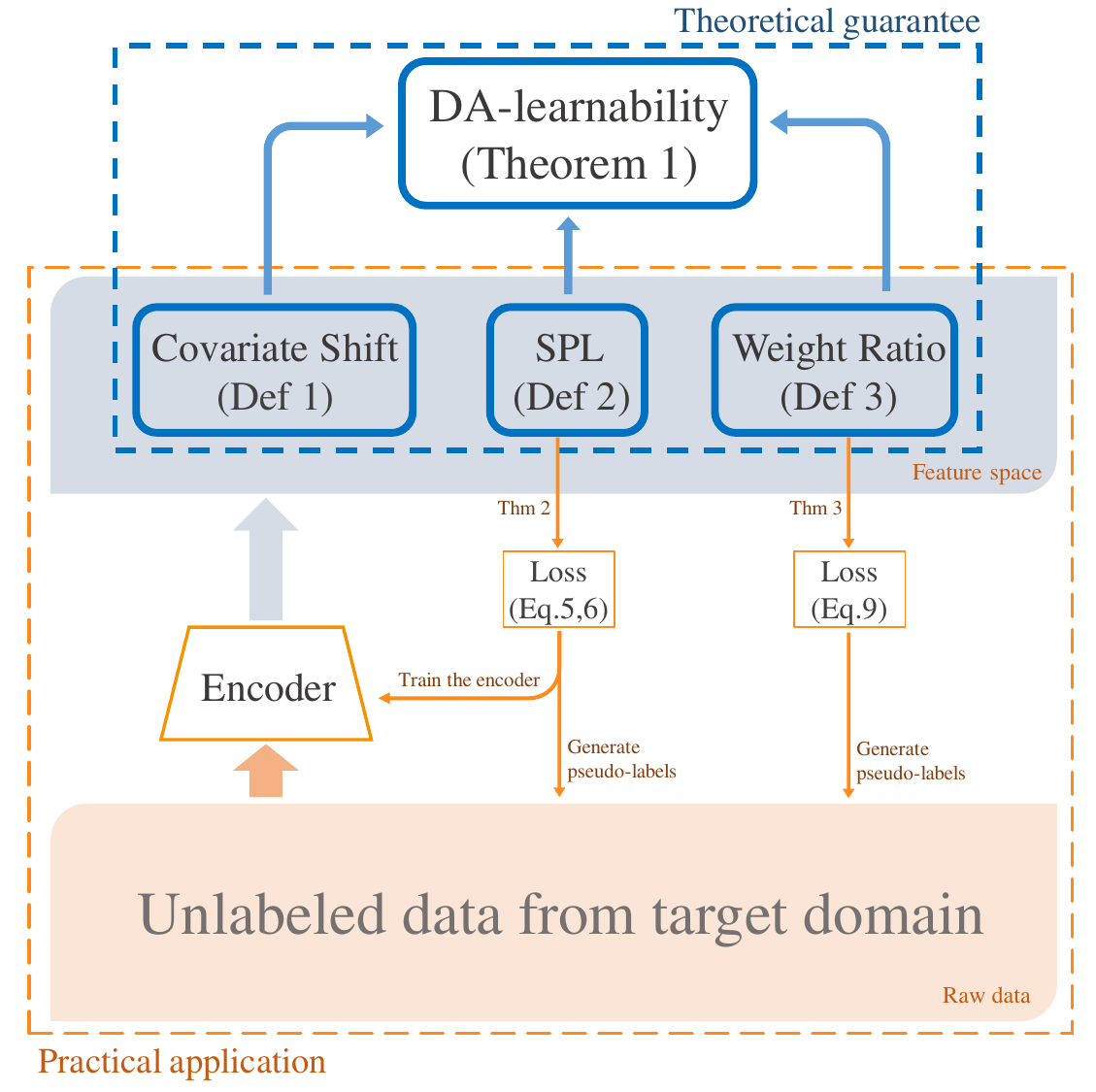}
	\caption{Structure of the paper.}
	\label{fig:structure}
\end{figure}

\subsection{Visualization of datasets and results}
To understand the variations between different domains more clearly, Figure \ref{fig:dataset} presents some samples from the datasets used in our experiments. These datasets all have their own special characteristics. For instance, people riding a bicycle are common in Market-1501, while these people are rare in DukeMTMC-reID. More importantly, the images in these re-ID datasets are heavily related to the cameras, which means that the images contain information closely knitted together with the camera, such as background, viewpoints or lighting condition.

\begin{figure}
	\centering
	\subfloat[Sample images from Market-1501 \cite{Zheng2015Dec}]{\includegraphics[width=0.45\textwidth]{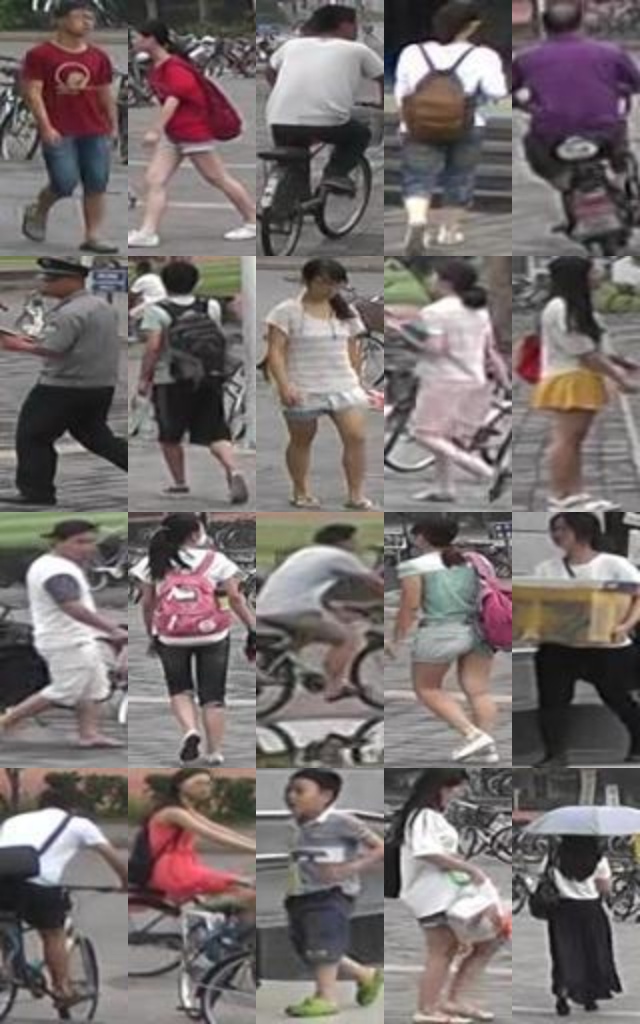}} \quad
	\subfloat[Sample images from DukeMTMC-reID \cite{Ristani2016}]{\includegraphics[width=0.45\textwidth]{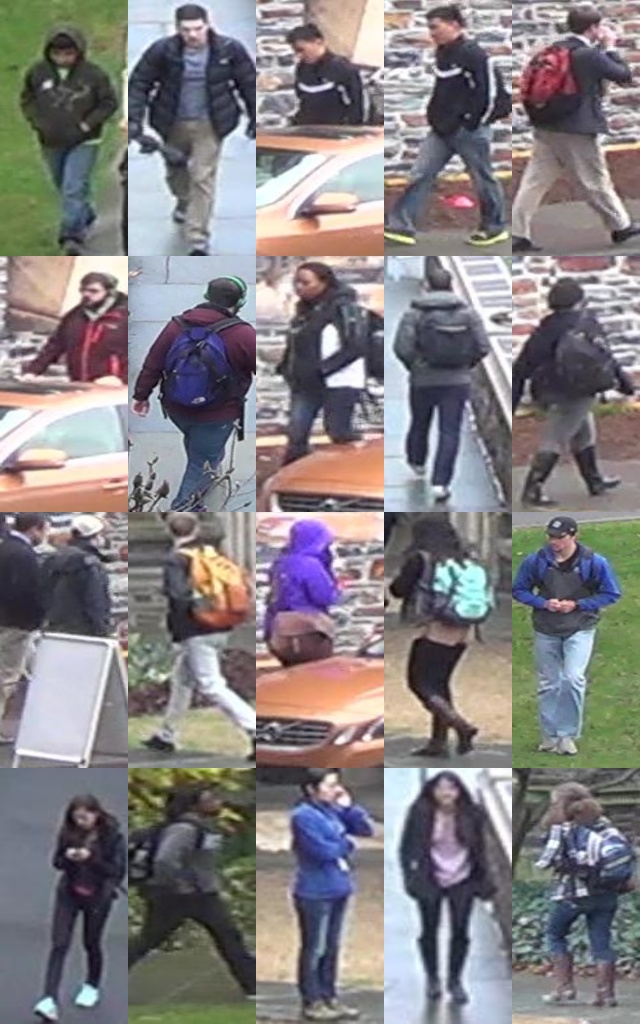}} \\
	\subfloat[Sample images from VeRi-776 \cite{Liu7553002}]{\includegraphics[width=0.45\textwidth]{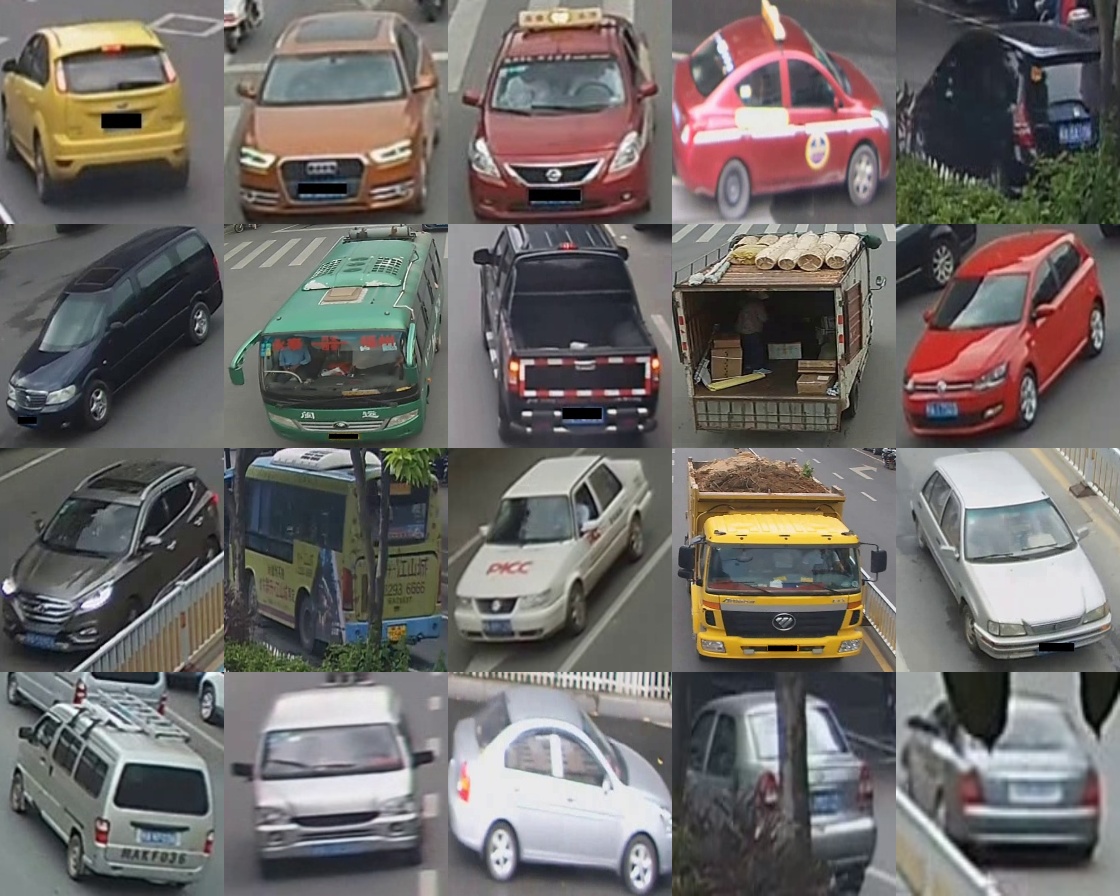}} \quad
	\subfloat[Sample images from PKU-VehicleID \cite{liu2016deep}]{\includegraphics[width=0.45\textwidth]{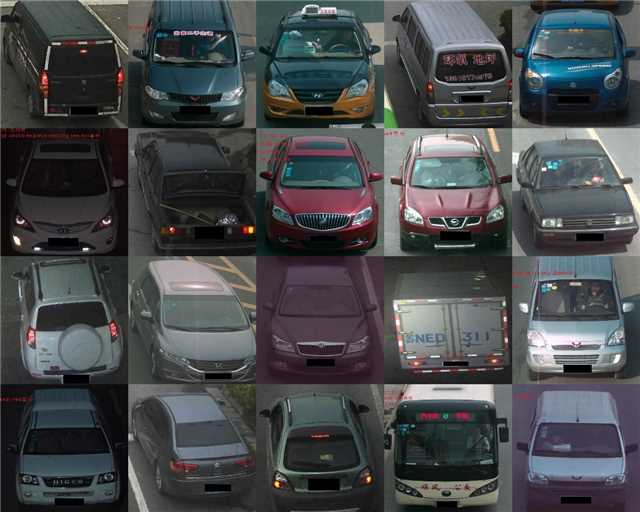}}
	\caption{Sample images from different datasets.}
	\label{fig:dataset}
\end{figure}

Moreover, we present some generated samples of SPGAN for vehicle re-ID. As shown in Figure \ref{fig:gen}, their image-image translation indeed works but fails to produce satisfactory re-ID results as person re-ID. This indicates that either their proposed generative method is not suitable for unsupervised domain adaptive vehicle re-ID, or their parameters need careful tuned for a new task. 

\begin{figure}
	\centering
	Original\hspace{4.5em}Generated\hspace{4.5em}Original\hspace{4.5em}Generated
	\includegraphics[width=0.8\textwidth]{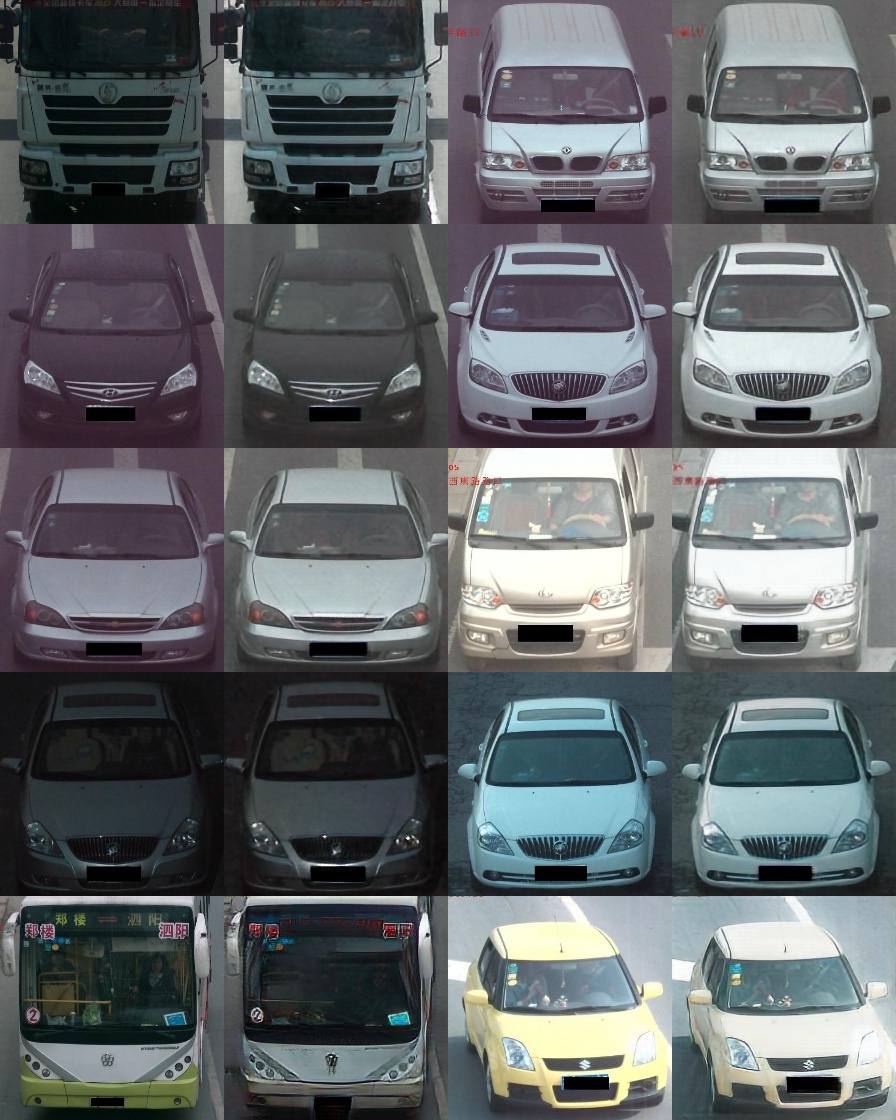}
	\caption{Sample images generated by SPGAN on vehicle datasets.}
	\label{fig:gen}
\end{figure}

\subsection{Encoder network}
Basically, the encoder network is ResNet-50 \cite{he2016deep} pre-trained on ImageNet and the whole network is presented in Figure \ref{fig:net}.

\paragraph{Person re-ID.} 
The size of input images is $256\times128\times3$, so the output of \texttt{conv5} is $8\times4\times2048$ and a average pool layer is added after \texttt{conv5} to have a output of size $1\times1\times2048$. We denote the output of this layer as \texttt{feat1}.
During training on the source domain, \texttt{feat1} is connected to a fully-connected (\texttt{fc}) layer with output 2048, denoted \texttt{fc0}, then the 2048 \texttt{fc} layer is connected to a \texttt{fc} layer with output 751 (Market-1501) or 702 (DukeMTMC-reID). Let the output of finally \texttt{fc} layer be \texttt{fc1}. The loss functions are \texttt{Softmax(fc1)} and \texttt{Triplet(feat1)}, which are added directly (without extra balancing parameter).
The model is trained by Adam optimizer \cite{Kingma2014Dec}.
Training parameters are set as follows: batch size 128 (PK sampling with P=16, K=8); maximum number epochs 70; learning rate 3e-4.

When training with data from target domain, there is no \texttt{fc1} layer and we use two triplet loss, that is \texttt{Triplet(feat1)} and \texttt{Triplet(fc0)}. The trick of using two triplet losses comes from \cite{Vo2018Mar}.
The model is trained by stochastic gradient descent and in each iteration step we perform data augmentation (random flip and random erasing) on the data.
Training parameters are set as follows: batch size 128 (PK sampling with P=32, K=4); momentum 0.9; maximum number epochs 70; learning rate 6e-5.
The networks are trained with two TITAN X GPUs.

\begin{figure}
    \centering
	\includegraphics[width=\textwidth]{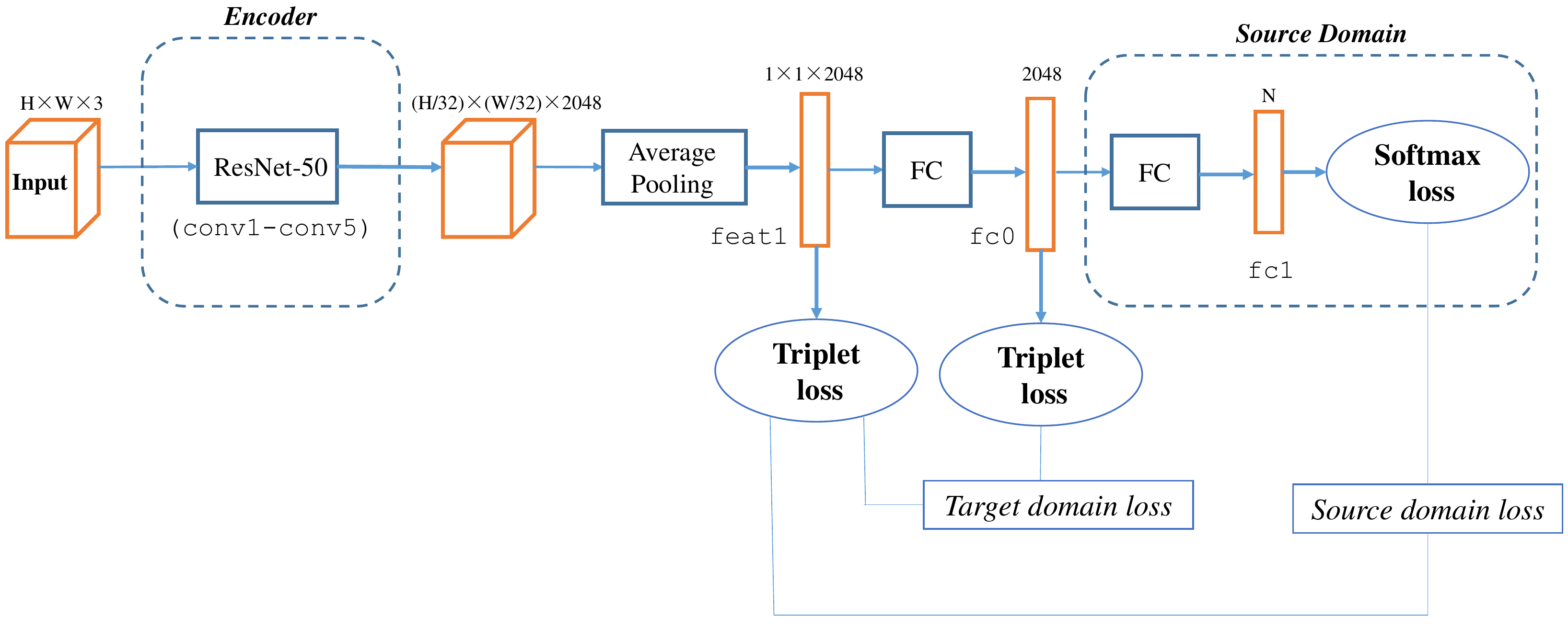}
	\caption{The architecture of our unsupervised domain adaptive network with ResNet-50 based encoder.}
	\label{fig:net}
\end{figure}

\paragraph{Vehicle re-ID.}
All parameters including network architecture are same as person re-ID, except the size of input data. The input data here are resized to $224\times224\times3$ and the output of \texttt{conv5} is $7\times7\times2048$.

\subsection{More results}

\paragraph{Effectiveness of $\diswr$.} \label{para:diswr}
From Table \ref{ta:person}, Table \ref{ta:vehicle} and Figure \ref{fig:conver}, we observe that in a practical view, using $\diswr$ actually is not appealing. We think the reasons are two folds. Firstly, the effectiveness of $\diswr$ depends on the distribution of source and target domain. In Fig.(\ref{fig:diswr}), we design a simple example in 2D feature space to show the validity of $\diswr$. In the left figure, the grey points denotes the extracted feature from source space and the colored points denotes the features of target data with real label. In the middle figure, we show the pseudo labels generated with \texttt{DBSCAN} when setting $\lambda=0$, i.e., not using $\diswr$. In the right figure, the results with $\diswr$ is shown. Comparing the middle figure and the right figure, we can see that $\diswr$ is important in such situation. The key idea in this demo is that those ``easy'' target points happen to be near the source data. Here, ``easy'' target points means the points belonging to the same ID are ``close'' in the extracted feature space with present encoder. This example can be also used for classification tasks since the Weight Ratio is a shared assumptions between our work and \cite{BenDavid2014}. Secondly, $\diswr$ is derived from $\lossfunwr$, but the potential value of $\lossfunwr$ is not fully exploited in our algorithm. Thus, using $\diswr$ in practical application is not appealing. However, when not using $\diswr$, the results are stable and good enough and already outperforms existing methods by large margin, which shows the power of the self-training scheme in domain adaptive re-ID problems. {\em For real applications, if the computation resources are limited, we recommend just setting $\lambda=0$ and not making the effort to search for an optimal $\lambda$.}
\begin{figure}
	\centering
    \includegraphics [width=0.32\textwidth]{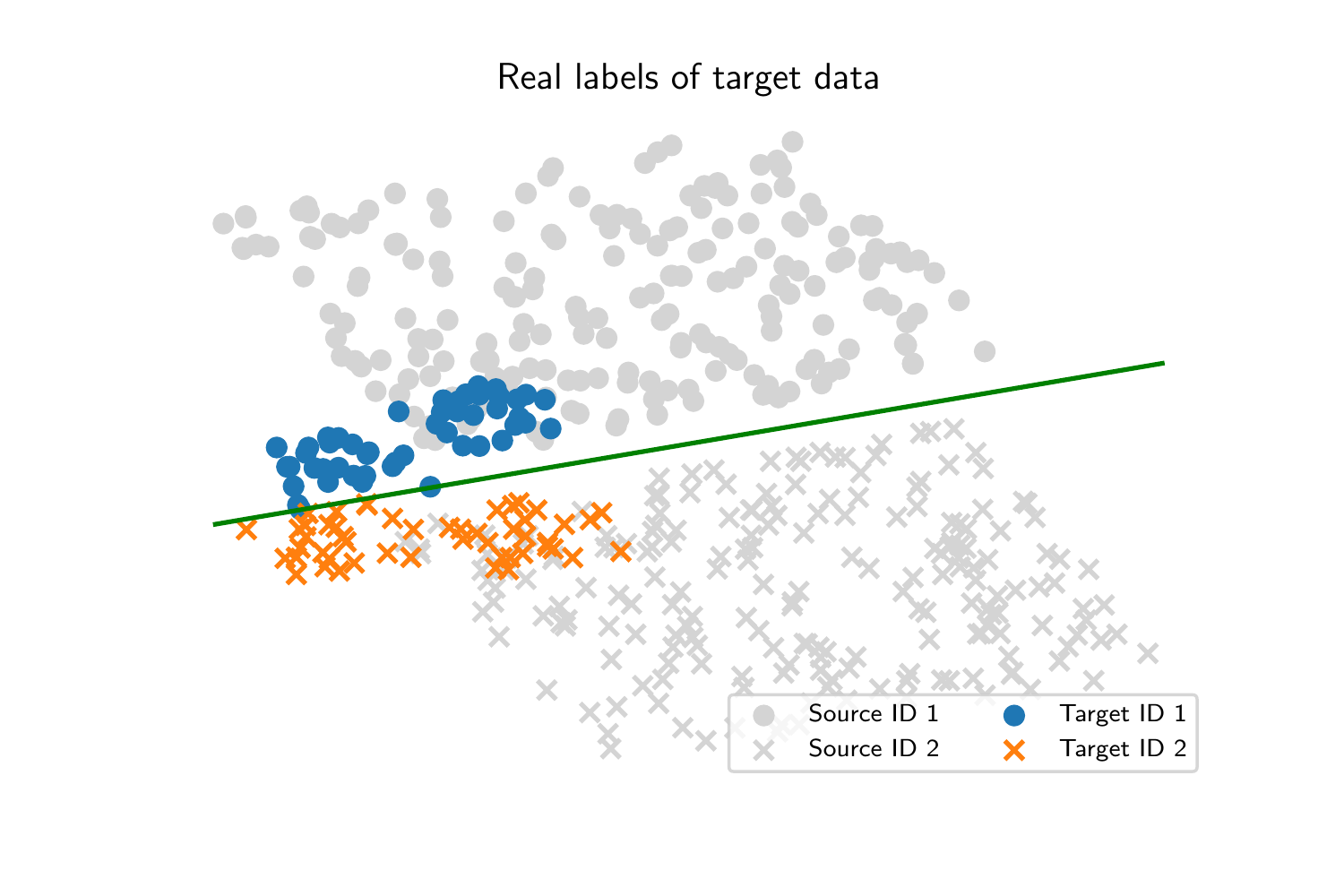}
    \includegraphics [width=0.32\textwidth]{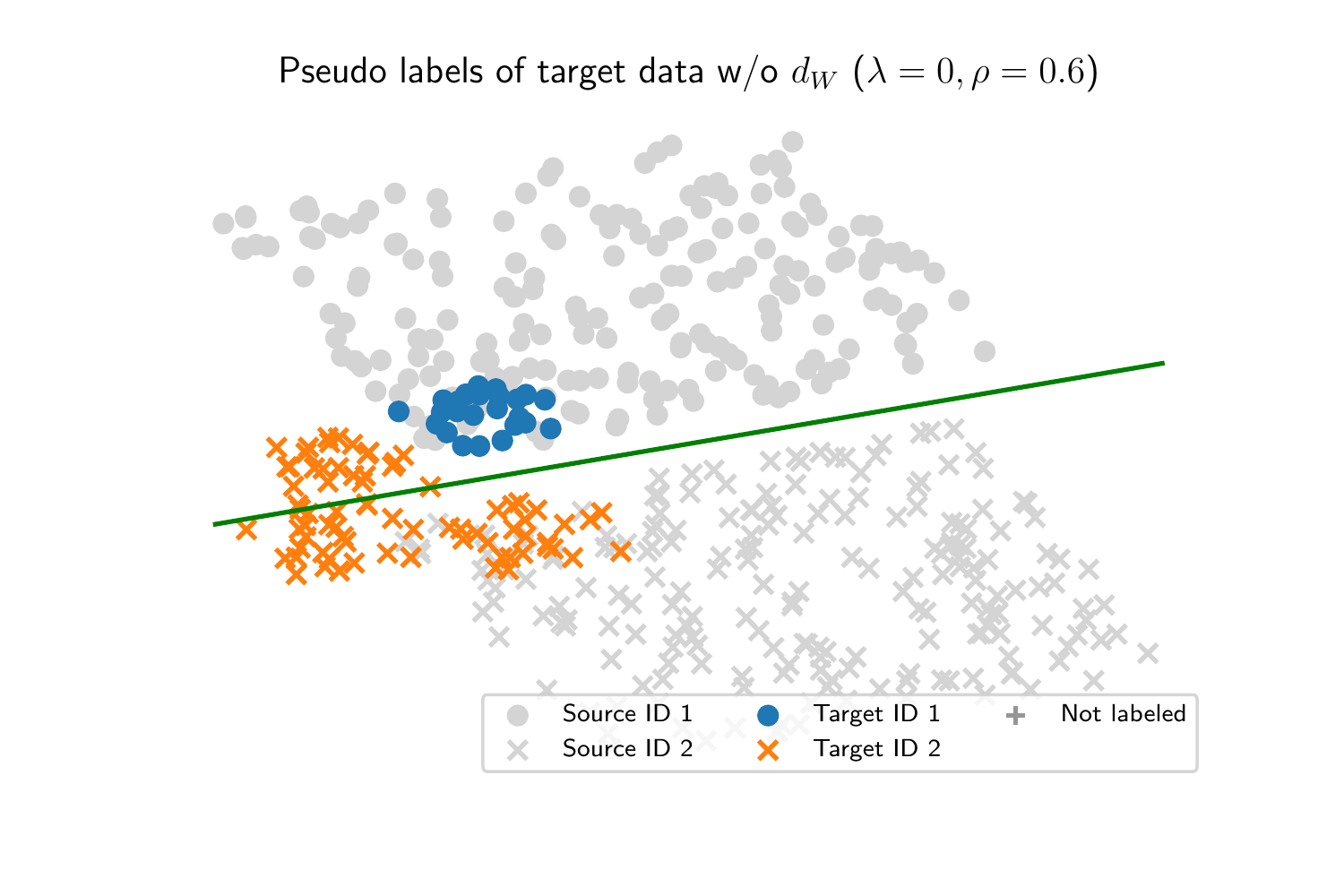}
    \includegraphics [width=0.32\textwidth]{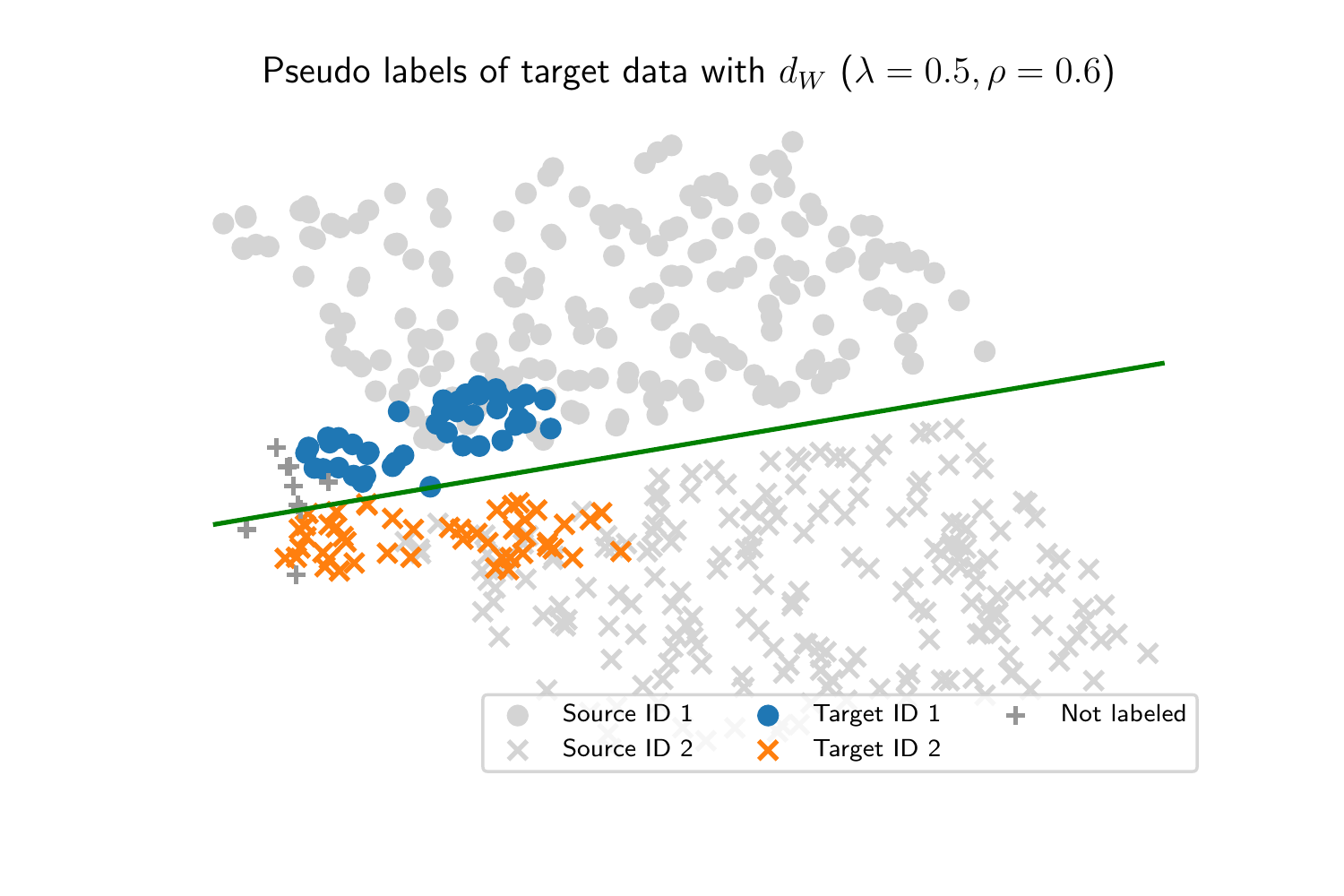}
	\caption{An example to show the effectiveness of $\diswr$.}
	\label{fig:diswr}
\end{figure}

\paragraph{Comparison of distance metrics.}
As for other contextual distance metrics, we test the performance of using the original Jaccard distance with or without $\diswr$ (also set $\lambda=0.1$).
For the Jaccard distance, we first compute the $k=20$ nearest neighbor set and then compute the distance between the sets. 
Another conclusion is that taking $\diswr$ into consideration is also beneficial for Jaccard distance. 
However, both of the two distance metrics are worse than the self-training baseline, i.e., Euclidean distance. The reason is that the Jaccard distance only consider the nearest neighbor sets and therefore pairs without overlapping nearest neighbors will have a Jaccard distance 1, which is too strict to generate enough training pairs. The shortcoming also leads to a slow or even halted increasing of accuracy, for more details see the convergence comparison paragraph and Figure \ref{fig:conver}.
As is shown in Table \ref{ta:dis}, $k$-reciprocal encoding employed in our method positively improve the performance of plain Jaccard distance. 

\begin{table}
	\centering
	\caption{Comparison of distance metrics.}
	\label{ta:dis}
	\begin{tabular}{c|cccc}
		\hline\hline
		\multirow{2}{*}{Methods} & \multicolumn{4}{c}{DukeMTMC-reID$\rightarrow$Market-1501}  \\ \cline{2-5}
		                               & rank-1    & rank-5    & rank-10   & mAP       \\ \hline
		Jaccard distance               & 63.8      & 80.0      & 85.3      & 37.1      \\
		Jaccard distance with $\diswr$ & \bf{65.7} & \bf{81.2} & \bf{86.5} & \bf{38.1} \\
		\hline\hline
	\end{tabular}
\end{table}

\paragraph{Comparison of clustering methods.}
Due to the restrictions of a suitable clustering method, we only test a version with affinity propagation \cite{frey2007clustering}. 
For task DukeMTMC-reID$\rightarrow$Market-1501, we investigate the effectiveness of affinity propagation with other distance metrics.
It is obvious that affinity propagation is not a proper clustering method for the reason that all data are used for clustering, which means it cannot avoid those pairs of low confidence.
As shown in Table \ref{ta:clu}, a interesting fact is that with affinity propagation just using Euclidean distance is better than our proposed distance. The reason behind this phenomenon is that the number of IDs (clusters) generated by affinity propagation is much larger when using our proposed distance. In Figure \ref{fig:ids}, we show the number of IDs with respect to each iteration step. Using our distance leads to a larger number of clusters out of the reason that our distance will enlarge the gap between the dissimilar pairs, which is ought to be beneficial of getting rid of these helpless stray samples. However, affinity propagation is a clustering method that every sample is assigned to some cluster and therefore using our distance performs worse than Euclidean distance. 

\begin{table}
	\centering
	\caption{Comparison of clustering methods.}
	\label{ta:clu}
	\begin{tabular}{c|cccc}
		\hline\hline
		Distance Metrics  & rank-1    & rank-5    & rank-10   & mAP       \\ \hline
		Euclidean         & \bf{63.5} & \bf{76.6} & \bf{80.7} & \bf{36.9} \\
		Ours w/o $\diswr$ & 62.4      & 74.6      & 78.9      & 35.2      \\
		Ours              & 62.4      & 74.5      & 78.8      & 35.6      \\
		\hline\hline
	\end{tabular}
\end{table}

\begin{figure}
	\centering
	\includegraphics[width=0.5\textwidth]{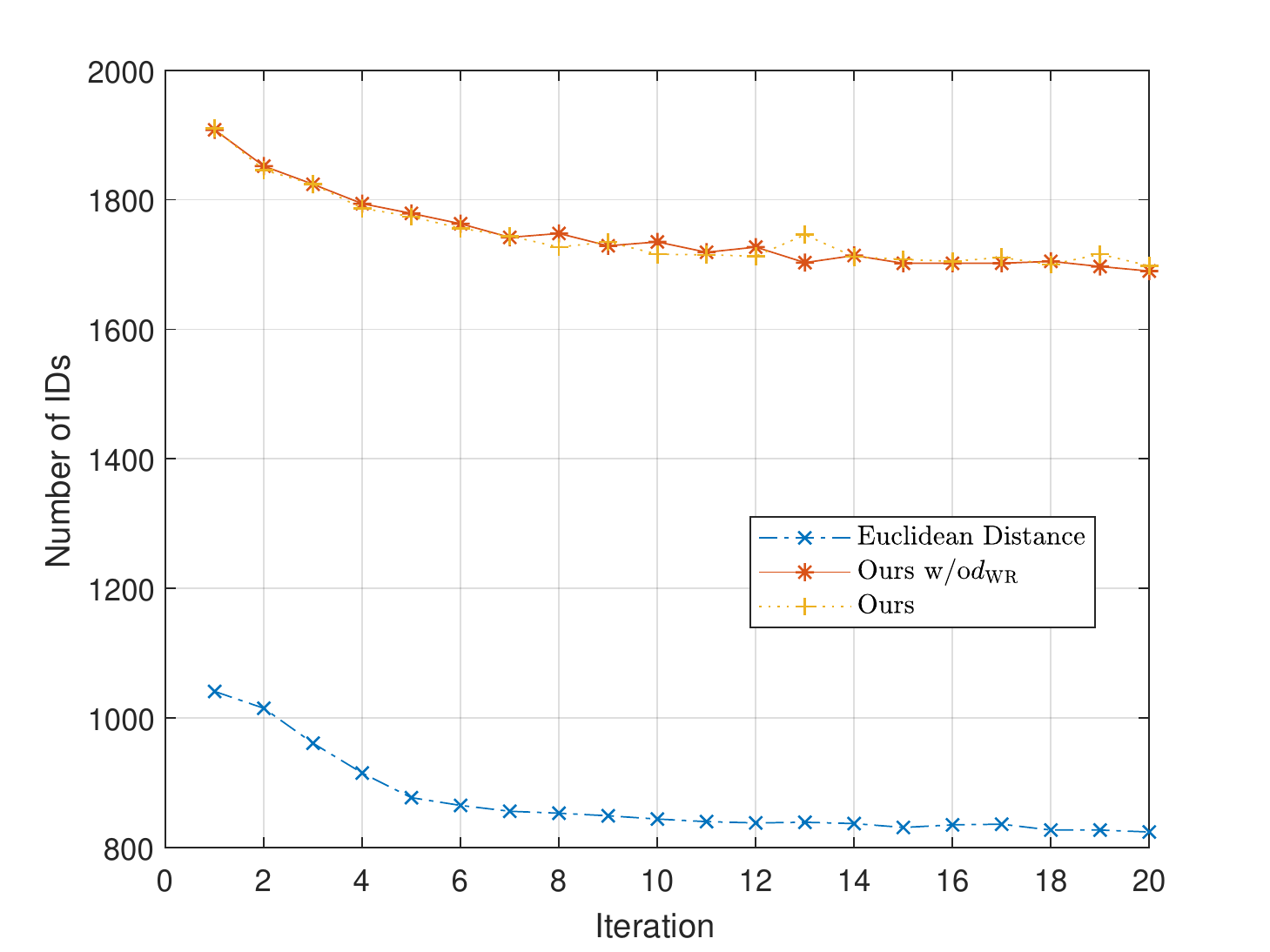}
	\caption{The number of IDs (clusters) on the target dataset of each iteration step when using affinity propagation.}
	\label{fig:ids}
\end{figure}

\paragraph{Parameters analysis}
Among all the parameters in our algorithm, the most influencing parameters are the percentage $p$ and the balancing parameter $\lambda$. 
Since the influence of $\lambda$ has been reported, here we perform experiments with a series of different $p$ from DukeMTMC-reID to Market-1501 and the results are shown in Table \ref{ta:plam}. 
As we can see from the table, even a small change ($2\times10^{-4}$) of $p$ has a discernible impact on the final accuracy. It is because that we use large scale datasets and the number of all possible pairs from target datasets is large. Take Market-1501 as an example, the number of training images is 12,936, so the number of all data pairs is over $8\times10^{7}$. Thus a small change of $p$ can cause a large change of the threshold.

\begin{table}
	\centering
	\caption{The impact of the parameter $p$ on person re-ID (from DukeMTMC-reID to Market-1501).}
	\label{ta:plam}
	\begin{tabular}{c}
		\begin{tabular}{c|cccc}
			\hline\hline
			                        & rank-1    & rank-5    & rank-10   & mAP       \\ \hline
			$p=1.0\times10^{-3}$ & 72.7      & 87.4      & 91.7      & 49.0      \\
			$p=1.2\times10^{-3}$ & 73.3      & 86.0      & 89.5      & 49.6      \\
			$p=1.4\times10^{-3}$ & 74.2      & 88.0      & 92.1      & 50.8      \\
			$p=1.6\times10^{-3}$ & \bf{75.8} & \bf{89.5} & \bf{93.2} & \bf{53.7} \\
			$p=1.8\times10^{-3}$ & 75.7      & 89.1      & 92.8      & 52.2      \\
			$p=2.0\times10^{-3}$ & 75.1      & 88.7      & 92.3      & 51.6      \\
			$p=2.2\times10^{-3}$ & 72.9      & 87.4      & 91.7      & 49.2      \\
			\hline\hline
		\end{tabular}
	\end{tabular}
\end{table}

\paragraph{Convergence comparison.}

In Figure \ref{fig:conver_all}, we use DukeMTMC-reID as source domain and Market-1501 as target domain and we first show the convergence results with different distance metric and clustering method in (a) and (b). 
Several conclusions can be drawn from the curves: First, we can see that the Jaccard distance based version becomes more stable after adding $\diswr$; Second, the accuracy of the Jaccard distance based version almost stops increasing after 14 iterations, which is caused by the special property of Jaccard distance mentioned before; Third, using affinity propagation converges very fast and after about 8 iterations the accuracy stop increasing, which is caused by the inaccurate number of clusters and all the samples are used to train the network. Thus the loss functions fail to be minimized through sample selection step.
Moreover, we show the results with different $p$ in (c) and (d). It is obvious that all the curves have a similar convergence tendency, which demonstrates that our iteration process is robust with regard of the crucial parameter $p$.

\begin{figure}
    \centering
	\subfloat[Rank-1 curves of different methods]{\includegraphics[width=0.49\textwidth]{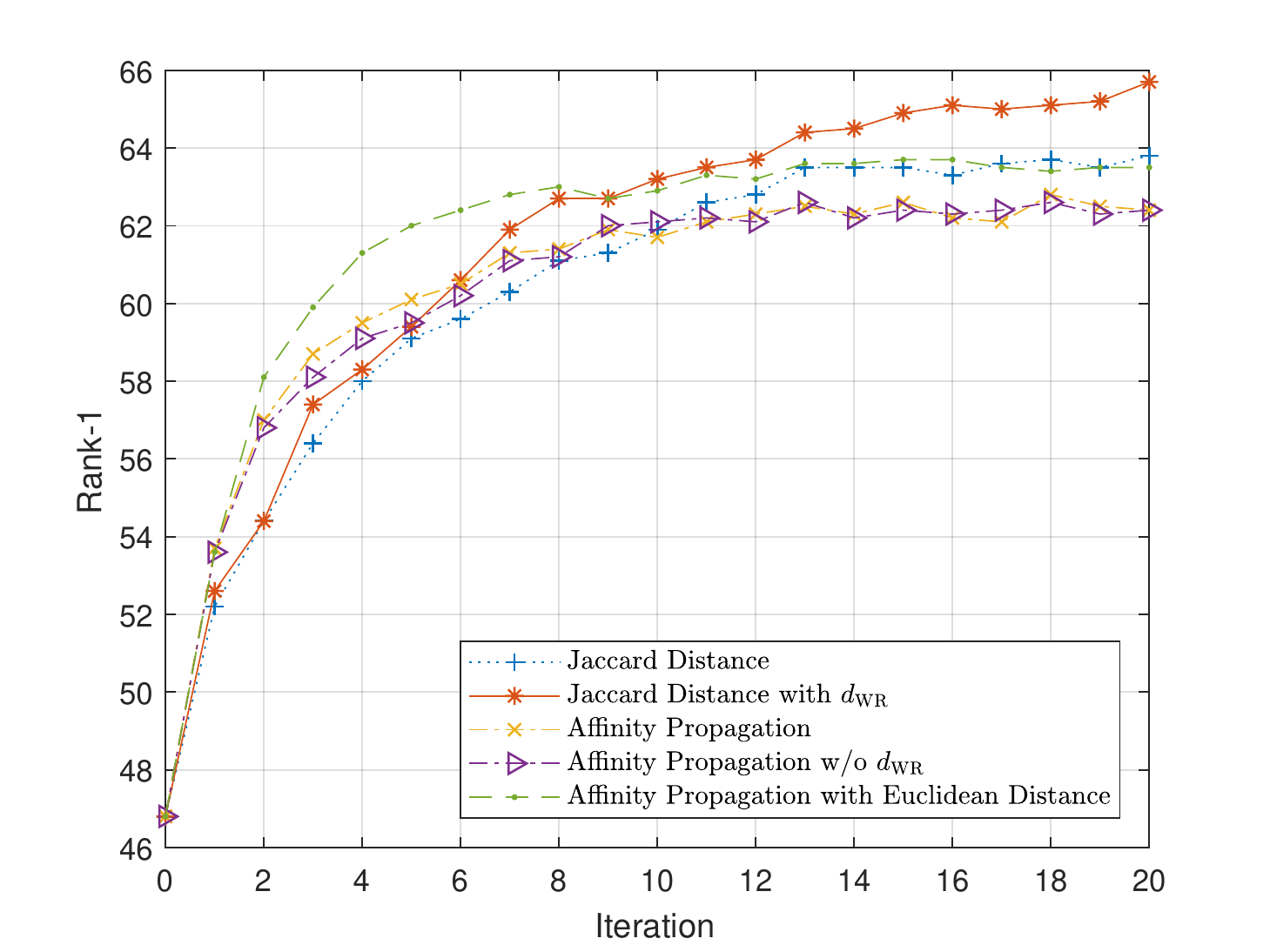}} 
	\subfloat[mAP curves of different methods]{\includegraphics[width=0.49\textwidth]{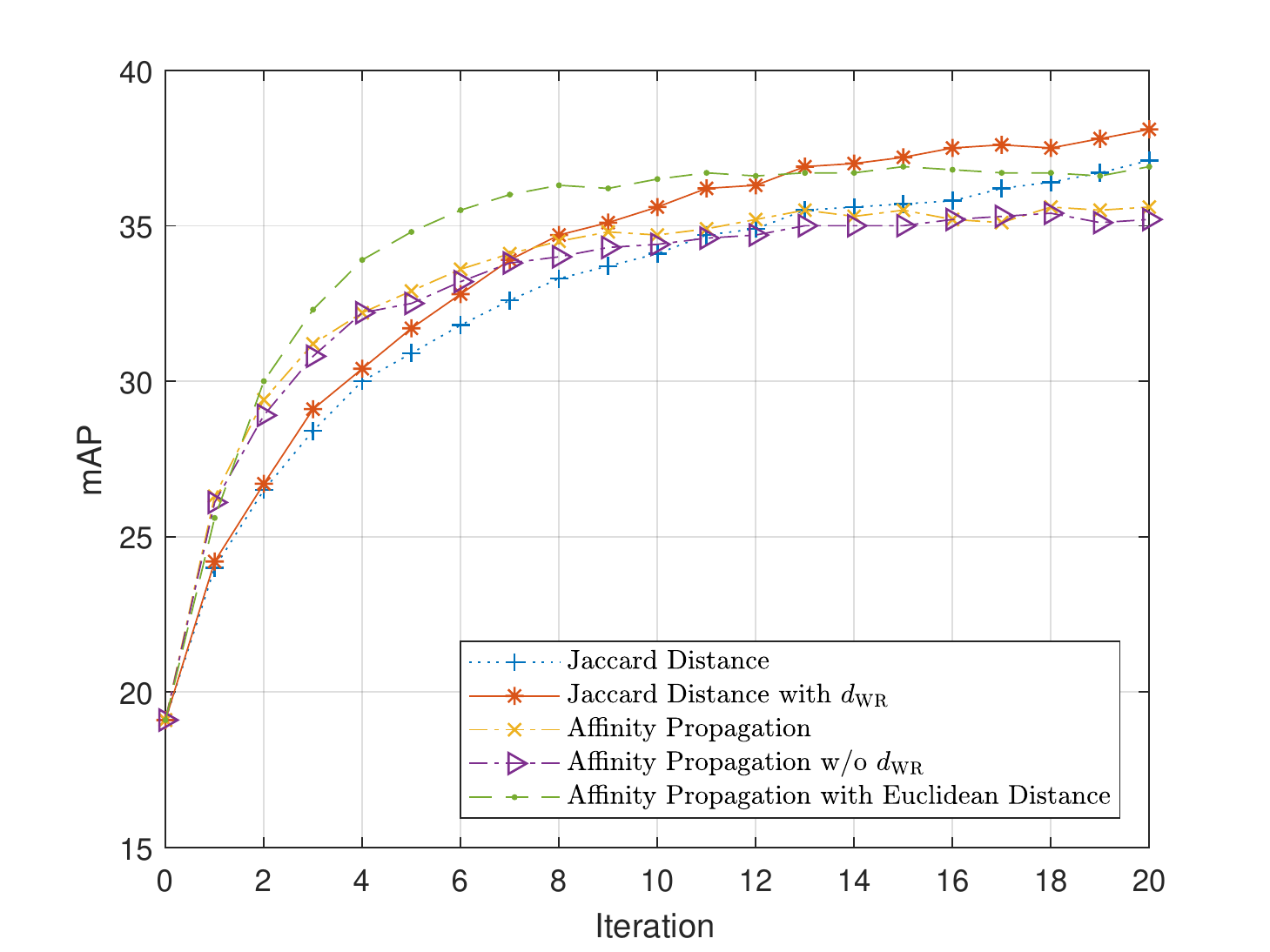}} \\
	\subfloat[Rank-1 curves with different $p$]{\includegraphics[width=0.49\textwidth]{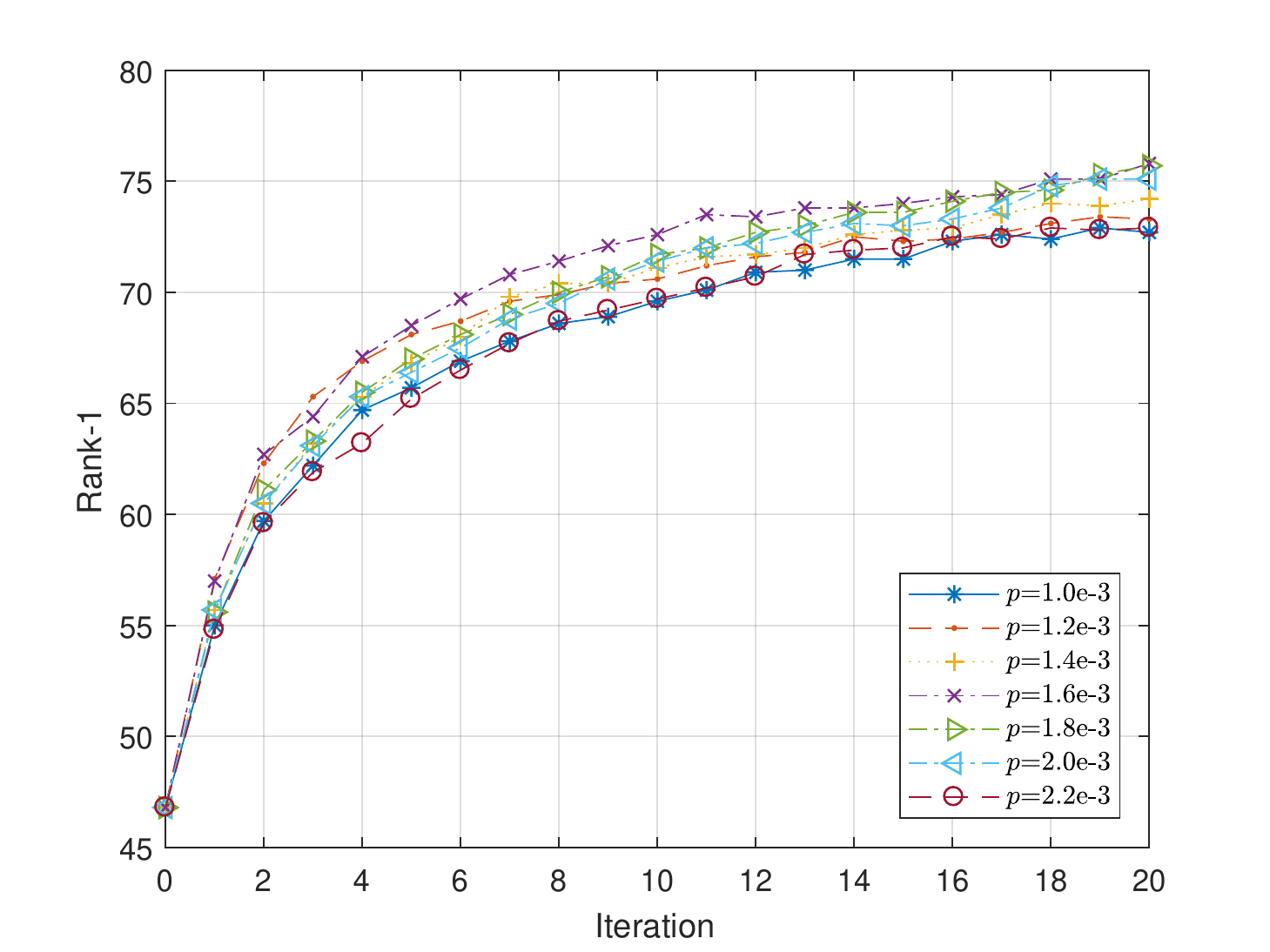}} 
	\subfloat[mAP curves with different $p$]{\includegraphics[width=0.49\textwidth]{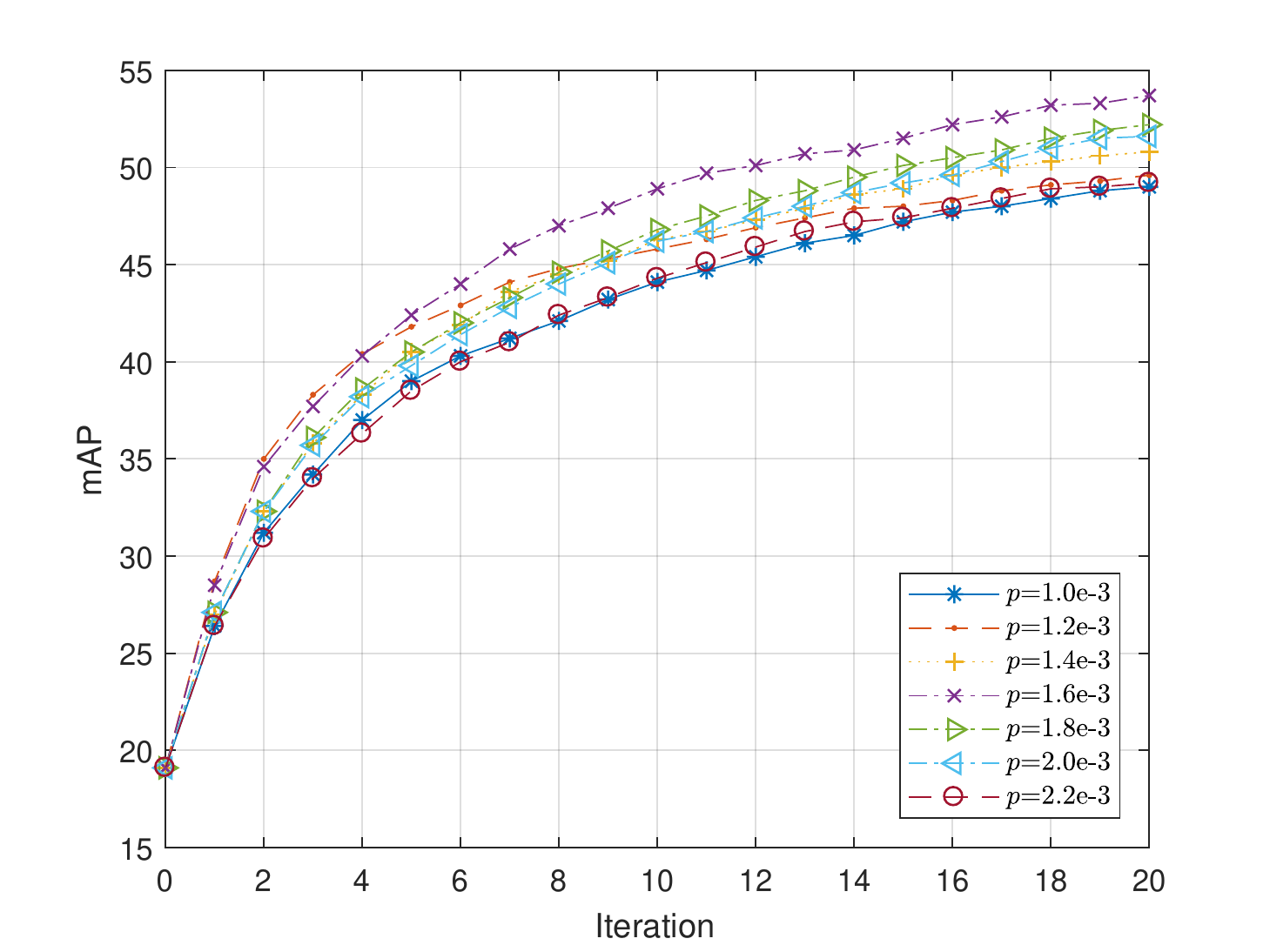}}
	\caption{Convergence comparison of different versions. We use DukeMTMC-reID as source domain and Market-1501 as target domain.}
	\label{fig:conver_all}
\end{figure}

\end{document}